%% file: main.tex
\documentclass{article}
\usepackage{float}
\usepackage{iclr2027_conference,times}

\input{math_commands.tex}

\usepackage{wrapfig}
\usepackage{hyperref}
\usepackage{url}
\usepackage{tikz}
\usepackage{algorithm}
\usepackage{algorithmic}
\usepackage{siunitx}
\usepackage{subcaption}
\usepackage[table]{xcolor}
\usepackage{booktabs}
\usepackage{siunitx}
\usepackage{thmtools}
\usepackage{thm-restate}

\usepackage{enumitem}
\usepackage[symbol]{footmisc}

\title{\textbf{SMat-Attention}: Structured Long-Context\\ Sequence Modeling}

\iclrfinalcopy

\author{
Emile Anand* \\
Georgia Institute of Technology \\
\texttt{emile@gatech.edu}
\And
Abdullah Ateyeh* \\
University of California, Berkeley \\
\texttt{abdullah\_ateyeh@berkeley.edu}
\And
Archer Wang* \\
Massachusetts Institute of Technology \\
\texttt{archerw@mit.edu}
\And
Marin Soljačić \\
Massachusetts Institute of Technology \\
\texttt{soljacic@mit.edu}
}

\begin{document}

\maketitle
\lhead{}

 \begin{abstract}
 
Long-context sequence models face a fundamental tradeoff: softmax attention uses flexible token-level interactions at quadratic cost, whereas linear attention obtains linear-time training and constant-time decoding by compressing history into a fixed-size state. In this work, we ask whether we can connect these regimes through a tunable notion of structure. To this end, we introduce Structured Matrix Attention (SMat-Attention) via a family of causal masks with structured long-range routing whose row supports have VC-dimension $d$. In our construction, $d=1$ recovers the standard causal mask, and increasing $d$ permits richer subset-routing patterns. We give chunkwise forward and backward algorithms to enable hardware-efficiency. For sequences of length $T$, the hard-routing construction takes $O(T^{2-3/d}+T)$ work, despite the mask being dense, for our prescribed family. In fixed-horizon streaming, decoding after the distant prefix takes constant time per token using $O(T^{1-1/d})$ cached states. SMat-Attention therefore makes VC-dimension an explicit knob governing access-pattern complexity, prefill cost, and decoding memory. Empirically, subset-routing and rule-assisted multi-key retrieval experiments illustrate the masks’ routing expressiveness. Extensions to Mamba-2 and Gated DeltaNet using learned routing with top-$k$ query reads retain subquadratic prefill, improve recall accuracy over the backbones in several settings, and achieve comparable small-scale language-modeling performance.
 \end{abstract}

\section{Introduction}

Attention is a foundational building block of modern deep learning \citep{bahdanau2016neuralmachinetranslationjointly} and serves as the core mechanism for modeling token interactions in Transformer architectures \citep{vaswani2023attentionneed}. Given key, query, and value matrices $\mQ, \mK$, and $\mV$, softmax attention computes
\begin{equation}\mathrm{Attention}(\mathbf{Q},\mathbf{K},\mathbf{V}) = \mathrm{softmax}\left(\frac{\mathbf{Q} \mathbf{K}^\top}{\sqrt{d_k}}\right)\mathbf{V}.\end{equation}
This operation gives each query direct access to token-level information, but its prefill computation grows quadratically with sequence length and its decoding cache grows linearly \citep{vaswani2023attentionneed}. However, fundamentally, long-context sequence modeling requires retaining useful information and selecting which parts of the past should influence each query. Hardware-aware kernels improve execution efficiency \citep{dao2023fa,shah2024flashattention3,liu2024ringattention,kwon2023efficient}, while sparse methods such as Native Sparse Attention and MoBA reduce the interactions evaluated for each query \citep{yuan2025native,lu2025moba}. \looseness=-1

Recurrent alternatives such as linear attention address these costs by compressing the history into a fixed-size recurrent state \citep{katharopoulos2020transformersrnnsfastautoregressive}. Modern variants improve how the model maintains this state. Structured state-space models (SSMs) \citep{fu2023hungryhungryhipposlanguage,gu2022efficientlymodelinglongsequences} compress history with time-invariant recurrences; Mamba and Mamba-2 make this recurrence input-dependent through selective gating \citep{gu2023mamba,dao2024transformers}, while DeltaNet and Gated DeltaNet use structured transition matrices \citep{schlag2021linear,yang2024gated,yang2025gateddeltanetworksimproving} that update via the delta rule \citep{10.1162/neco.1992.4.1.131,65669.104390}. These mechanisms improve retention and retrieval; however, their fixed-size hidden state still constrains associative recall over long contexts \citep{arora2023zoologymeasuringimprovingrecall}.\looseness=-1

\noindent These advances highlight the role of structure in efficient sequence modeling. For instance, linear attention exploits its causal prefix structure to reuse accumulated key–value summaries, yielding $O(T)$ computation. Gated variants, in turn, extend this approach through semiseparable structure ~\citep{dao2024transformers}, while long-convolution models exploit Toeplitz structure to compute their outputs in $O(T\log T)$ time using FFT ~\citep{poli2023hyena, qin2023toeplitzneuralnetworksequence}. Log-Linear Attention \citep{guo2025loglinearattention} further expands this design space by changing the organization of memory: it organizes recurrent summaries through a Fenwick-tree hierarchy, achieving $O(T\log T)$ computation and $O(\log T)$ decoding memory. 
Other recent approaches route tokens among multiple recurrent states \citep{du2026mom} or let the compressed memory grow with context length \citep{behrouz2026memorycaching, goldstein2026kvm}, or as a latent vector \citep{anand2026continuouslatentcontextsenable}. These approaches motivate studying not only how much memory a model retains, but also which subsets of stored information each query can access. Therefore, we investigate the following question: \emph{can the combinatorial richness of long-range access patterns be made an explicit architectural parameter, with corresponding guarantees on computation and memory?}

We study structured long-range access as an intermediate regime and how its complexity governs computation and memory. To make this complexity explicit, we use the VC dimension of a causal mask’s row supports \citep{kearns1994introduction,vapnik1971uniform}. Each row specifies the keys available to a query; under this lens, the VC dimension measures the largest number of keys on which the rows realize every possible subset. Recent connections between VC dimension and matrix multiplication \citep{anand2025structuralcomplexitymatrixvectormultiplication} motivate constructing attention mechanisms that couple this combinatorial parameter to computational guarantees.\looseness=-1

We introduce Structured Matrix Attention ({SMat-Attention}), a family of causal attention masks built from point-hyperplane incidences over finite fields, whose row-support VC dimension \(d\) explicitly controls long-range access complexity. Our framework recovers ordinary causal masking when \(d=1\), and exploits additional structure to support richer access patterns with subquadratic attention. Building on structured masked-attention formulations \citep{choromanski2022block}, we establish the following results:
\begin{enumerate}[leftmargin=*, labelsep=0.5em]
    \item  We exploit the resulting incidence structure to derive chunkwise forward and backward algorithms. For sequences of length $T$, we show that this mechanism takes $O(T^{2-3/d}+T)$ work, yielding $O(T)$-attention for $d\leq 3$ and $O(T^{2-3/d})$-attention for $d>3$, despite the full causal mask having $\Theta(T^2)$ nonzero entries.
    \item  We show that after processing a fixed distant prefix, SMat-Attention supports \(T\)-independent per-token decoding using \(O(T^{1-1/d})\) cached states, revealing an explicit tradeoff between long-range access complexity and memory.
\item We extend the construction to Mamba-2 and Gated DeltaNet with learned content hashing and a learned four-read selector. The extension inherits the tabulation, cache and VC bounds of the above; its selector adds $\Theta(T^{2-1/d})$ prefill work and $O(T^{1-1/d})$ per decoded token. Controlled tasks show benefits consistent with increased routing expressiveness; learned SMat extensions improve mean recall accuracy over native backbones in several tested settings and remain competitive on small-scale PG-19 language modeling.
\end{enumerate} 
\section{Preliminaries}

Let $T$ be the length of the input sequence. Following \cite{vaswani2023attentionneed}, attention linearly projects the input tokens into $\mathbf Q\in\R^{T\times d_{QK}}$, $\mathbf K\in\R^{T\times d_{QK}}$ and $\mathbf V\in\R^{T\times d_v}$, the queries, keys and values. Following \cite{choromanski2022block}, the \emph{general masked kernel attention} is
\[
  \mathsf{Att}_K(\mQ,\mK,\mV,\mM)=\mD^{-1}\mA\mV,
  \qquad
  \mA=\mM\odot\mathcal K(\mQ,\mK),
  \qquad
  \mD=\mathsf{diag}(\mA\mathbf 1_T),
\]
where $\odot$ is the entrywise product, $\mathcal K:\R^{d_{QK}}\times\R^{d_{QK}}\to\R$
is a kernel, $\mathcal K(\mQ,\mK)_{ij}=K(\mathbf q_i,\mathbf k_j)$ for
$\mathbf q_i$ the $i$th row of $\mQ$ and $\mathbf k_j$ the $j$th row of $\mK$, and \(\mathbf 1_T\) is the all-ones vector.
Softmax attention is the special case $K(x,y)=\exp(\frac{x^\top y}{\sqrt{d_{QK}}})$
where $\mM=\exp(\mN)$ entrywise and $\mN$ is the logits mask.\looseness=-1

\textbf{Finite-feature attention.}
Assuming the kernel has a nonnegative feature factorization
of dimension $r$, $\mathcal K(q,k)=\phi_\mQ(q)^\top\phi_\mK(k)$, where $ \phi_\mQ(q),\phi_\mK(k)\in\R^r_{\ge0}$.
Let $\phi_i:=\phi_\mQ(q_i)$ and $\psi_j:=h_j\phi_\mK(k_j)$, where $h_j\ge0$ is an
optional key gate (set $h_j\equiv1$ for ungated attention). The normalizer is
carried along with the values by appending a constant coordinate: with
$p:=d_v+1$,
\begin{equation}
  \bar v_j:= \begin{bmatrix} v_j \\ 1\end{bmatrix} \in\R^p,
  \qquad
  \mZ_j:=\psi_j\bar v_j^\top\in\R^{r\times p}.
  \label{eq:augmented-kv-state}
\end{equation}
For any nonnegative mask $\mM$ the \emph{augmented output} and the attention
output are
\begin{equation}
  \bar y_i=\phi_i^\top\sum_{j=1}^T \mM_{ij} \mZ_j\in\R^p,
  \qquad
  o_i=\frac{\bar y_{i,1:d_v}}{\bar y_{i,p}} .
  \label{eq:augmented-attention-output}
\end{equation}
Through the augmentation, $\bar y_i$ is \emph{linear} in $\mM$: the
last coordinate accumulates the denominator along with the numerator, and the single nonlinearity is the final division. So if $\smash{\mM=\sum_\ell \mM^{[\ell]}}$, the
contributions $\smash{\bar y_i^{[\ell]}}$ can be computed independently, in different
orders and with different computational kernels, as long as they are summed before the
division. Section~\ref{subsec:chunkwise} does this with two summands. Importantly, we require the feature map to be finite and
nonnegative, so softmax attention is covered only through a kernel approximation such as \cite{choromanski2022rethinkingattentionperformers}.\looseness=-1

\textbf{VC dimension of a mask.}
A binary mask $\mM\in\{0,1\}^{T\times T}$ defines a set system on the key indices: row $i$ is the set $S_i(\mM)=\{j\in[T]:\mM_{ij}=1\}$ of keys visible to query $i$. If we let $\mathcal S(\mM)=\{S_1(\mM),\dots,S_T(\mM)\}$, then $\mathrm{VC}(\mM)$ is the VC dimension of $\mathcal S(\mM)$, i.e.\ the largest $k$ for which some set of $k$ keys is shattered by the rows. The causal mask $\mathbf L_T$ has $\mathrm{VC}=1$: its rows are the prefixes
$\{1,\dots,i\}$, which are totally ordered, so no two keys can be shattered, i.e. no query sees a later key without also seeing every earlier one. On the other hand, an unconstrained $T$-row mask can have VC-dimension as large as $\lfloor\log T\rfloor$. The parameter $d$ interpolates between these, and Theorems~\ref{thm:exact}
and~\ref{thm:decode} price the interpolation.

\section{SMat-Attention}

We construct a family of causal masks
$\mathbf M^{(1)},\mathbf M^{(2)},\dots$ indexed by their VC dimension. We show that attention under $\mathbf M^{(d)}$ can be computed in $O(T^{2-3/d} + T)$ work. Each mask is built from incidences between points and
hyperplanes over a finite field, which gives it a computationally favorable structure. For $d\le 3$ the forward pass is linear in the sequence length, and decoding runs from a cache of $O(T^{1-1/d})$ states.

\label{subsec: mask constructions}
\begin{wrapfigure}[18]{r}{0.40\textwidth}
\vspace{-17pt}
\centering
\begin{tikzpicture}[scale=0.9]
  \draw[thick] (0,0) rectangle (5,5);
  \draw[thick] (2.2,0) -- (2.2,5);
  \draw[thick] (0,2.8) -- (5,2.8);
  \fill[blue!12] (0,5) -- (2.2,2.8) -- (0,2.8) -- cycle;
  \draw[thick] (0,5) -- (2.2,2.8);
  \node at (0.65,3.35) {\small $\mathbf{L}_n$};
  \node at (3.6,3.9) {\small $\mathbf{0}$};
  \fill[orange!16] (0,0) rectangle (2.2,2.8);
  \node at (1.1,1.3) {\small $\mathbf G=\mathbf R\,\mathbf C\,\mathbf E^\top$};
  \fill[green!14] (2.2,2.8) -- (5,0) -- (2.2,0) -- cycle;
  \draw[thick] (2.2,2.8) -- (5,0);
  \node at (3.1,1) {\small $\mathbf{L}_m$};
  \draw[<->] (0,5.35) -- (2.2,5.35) node[midway,above] {\scriptsize distant $[n]$};
  \draw[<->] (2.2,5.35) -- (5,5.35) node[midway,above] {\scriptsize recent $(n,T]$};
  \node[rotate=90] at (-0.45,3.9) {\scriptsize distant queries};
  \node[rotate=90] at (-0.45,1.4) {\scriptsize recent queries};
\end{tikzpicture}
\caption{Block layout of $\mathbf{M}^{(d)}$; shaded regions are nonzero. Both
diagonal blocks are ordinary causal masks, and all long-range structure lives in $\mathbf{G}$.}
\label{fig:blocks}
\end{wrapfigure}
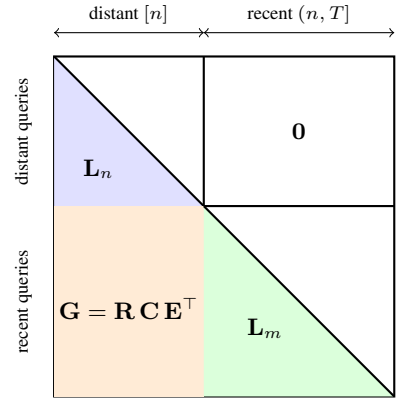

\textbf{Block form.}
Let $n$ be the number of \emph{distant} tokens, and $m=T-n$ be the remaining \emph{recent} tokens. Let $\mathbf L_s$ be the inclusive lower-triangular all-ones matrix of order $s$. Every mask in the family has the form
\begin{equation}
  \mathbf M^{(d)}
  =\begin{pmatrix}
     \mathbf L_n & \mathbf 0\\[2pt]
     \mathbf G   & \mathbf L_m
   \end{pmatrix},
  \qquad \mathbf G\in\{0,1\}^{m\times n},
  \label{eq:blockform}
\end{equation}
so distant and recent tokens are each causally masked, and $\mathbf G$ encodes all of the long-range interaction.

\textbf{The construction of $\mathbf{G}$.}
Let $d\geq 2$ and $q$ be a prime. The geometry is the ambient space of $(d-1)$-dimensional vectors over the finite field $\mathbb F_q$, given by $\mathbb F_q^{d-1}$, with its affine hyperplanes $H_{a,b}=\{x:a^\top x=b\}$, one per normalized direction $a\neq0$ and offset $b$, of which there are
\begin{equation}
  B=\frac{q(q^{d-1}-1)}{q-1}=\Theta\big(q^{d-1}\big).
  \label{eq:Bcount}
\end{equation}
Each distant key is assigned a \emph{profile}, a point of the ambient space, and each recent query a \emph{type}, one of the hyperplanes,
\[
  \operatorname{prof}:[n]\to\mathbb F_q^{d-1},
  \qquad
  \operatorname{type}:[m]\to\{0,\dots,B-1\},
\]
and a recent query attends to a distant key when that key's profile lies on the query's hyperplane:
\begin{equation}
  \mathbf G_{ij}=\mathbbm 1\{\operatorname{prof}(j)\in H_{\operatorname{type}(i)}\},
  \label{eq:Gdef}
\end{equation}
i.e. recent query of type $h$ attends to its own prefix among the recent tokens and every distant token whose profile lies on $H_h$. $\mathbf G$ is the only component of the mask which is $d$-dependent and its structure dictates the VC-dimension (Theorem \ref{thm:props}). In order to establish a VC lower bound for Theorem \ref{thm:props}(iii), we further characterize $\operatorname{prof}$ and $\operatorname{type}$. Let $e_1,\dots,e_{d-1}$ be the standard basis of $\mathbb F_q^{d-1}$ and, for $R\subseteq[d-1]$, let $H_R$ be the witness hyperplane. We impose the condition that there are distinct distant positions $j_1,\dots,j_{d-1}$ with
$\mathrm{prof}(j_\ell)=e_\ell$, and a recent index $\tau$ such that,
for every $R\subseteq[d-1]$,  $H_R$ occurs at indices
$i_R^-<\tau\le i_R^+$. We give two concrete examples:
\begin{enumerate}[leftmargin=*, labelsep=0.5em]
    \item The \emph{positional} assignment takes $\operatorname{prof}(j)$ to be the base-$q$ digits of $(j-1)\bmod q^{d-1}$ and $\operatorname{type}(i)=i\bmod B$.
    \item The \emph{content-based} assignment fixes a hash $x$ which maps each token to a point of $\mathbb F_q^{d-1}$. This mapping can be fixed or learned. Let $u_t$ be the hashed vector at position $t$. A distant key has $\operatorname{prof}(j)=u_j$ and for a hashed direction $a$, a recent query has $\operatorname{type}(i)=H_{a,a^\top u_{n+i}}$ (the hyperplane through its own cell). A query sees every distant key whose token repeats its own.
\end{enumerate}  

\textbf{Scaling with sequence length.} We define the family by prescribing the field size as a function of sequence length and VC-dimension. For each fixed $d\geq 2$, we choose a prime $q=\Theta(n^{1/d})$, where $n,m=\Theta(T)$. Writing $P=q^{d-1}$ for the number of profile cells, this gives $P=\Theta(T^{1-1/d})$ and $B=\Theta(T^{1-1/d})$. For each fixed $d$, a sufficiently large $T$ ensures $P\leq n$ and $2B\leq m$, as required by our positional construction. Moreover, for $d=1$, when $\mathbb F_q^{0}$ is a single point, we use $B=1$ and $H_0=\mathbb F_q^{0}$. Then $\mathbf G$ is all ones and the block form becomes $\mathbf M^{(1)}=\mathbf L_T$, ordinary causal masking.

\begin{restatable}[Properties of hard-routing SMat masks $\mathbf M^{(d)}$]{thm}{propertiesofsmatmasks}
\label{thm:props}
Let $T\ge 2$ and $1\le d<\lfloor\log_2T\rfloor$. Then \emph{(i)} $\mathbf M^{(d)}\in\{0,1\}^{T\times T}$ is causal with
$\mathbf M^{(d)}_{tt}=1$ for all $t$;
\emph{(ii)} $\mathbf M^{(1)}=\mathbf L_T$, 
\emph{(iii)} $\mathrm{VC}(\mathbf M^{(d)})=d$; and
\emph{(iv)} the number of nonzero entries of $\mathbf M^{(d)}$, given by $\mathrm{nnz}(\mathbf M^{(d)})$, satisfies $\mathrm{nnz}(\mathbf M^{(d)})=\Theta(T^2)$.
\end{restatable}

\subsection{Chunkwise SMat-attention}
\label{subsec:chunkwise}

\textbf{Splitting the mask.}
Split the block form into
\begin{equation}
  \mathbf M^{(d)}
  =\underbrace{
    \begin{pmatrix}\mathbf L_n&\mathbf 0\\\mathbf 0&\mathbf L_m\end{pmatrix}
   }_{\text{two independent causal masks}}
  +\underbrace{
    \begin{pmatrix}\mathbf 0&\mathbf 0\\\mathbf G&\mathbf 0\end{pmatrix}
   }_{\text{long range}} .
  \label{eq:factor}
\end{equation}
Since $\bar y_i$ is linear in the mask, the two attention branches can be computed separately and added before the final division. By construction $\mathbf G_{ij}$ depends on $i$ only through
$\operatorname{type}(i)$ and on $j$ only through $\operatorname{prof}(j)$. It
therefore factors through the point--hyperplane incidence matrix
\begin{equation}
  \mathbf C\in\{0,1\}^{B\times q^{d-1}},
  \qquad
  \mathbf C_{hx}=\mathbbm 1\{x\in H_h\},
  \qquad
  \mathbf G_{ij}=\mathbf C_{\operatorname{type}(i),\operatorname{prof}(j)} .
  \label{eq:C}
\end{equation}
Writing $\mE\in\{0,1\}^{n\times q^{d-1}}$ and $\mR\in\{0,1\}^{m\times B}$ for the
one-hot matrices of $\operatorname{prof}$ and $\operatorname{type}$, $\mathbf G=\mR\,\mathbf C\,\mE^\top$. Distant keys sharing a profile can be pooled once and reused by every query that sees that profile, and recent queries sharing a type can be answered from one aggregated state.

\textbf{The long-range branch.}
We apply the factorization $\mathbf G=\mR\,\mathbf C\,\mE^\top$ one factor at a
time. $\mE^\top$ pools the distant states into a \emph{profile table}
$\{F_x\}$, one state per profile; $\mathbf C$ converts this into a
\emph{type table} $\{U_h\}$, one state per query type; $\mR$ contracts
each query against the entry for its own type:
\begin{equation}
  F_x:=\!\!\!\sum_{j\le n:\,\operatorname{prof}(j)=x}\!\!\!\mZ_j,
  \qquad
  U_h:=\sum_x\mathbf C_{hx}F_x,\label{eq:schedule}
  \end{equation}
  where $\bar y^{\mathrm{lr}}_{n+i}=\phi_{n+i}^\top U_{\operatorname{type}(i)}$ for $i\in[m]$, and $\bar y^{\mathrm{lr}}_i=0$ for $i\leq n$.
Queries of a common type are contracted together in one matrix product and
their rows scattered back into chronological order.

\begin{figure}
    \centering
    \includegraphics[width=0.75\linewidth]{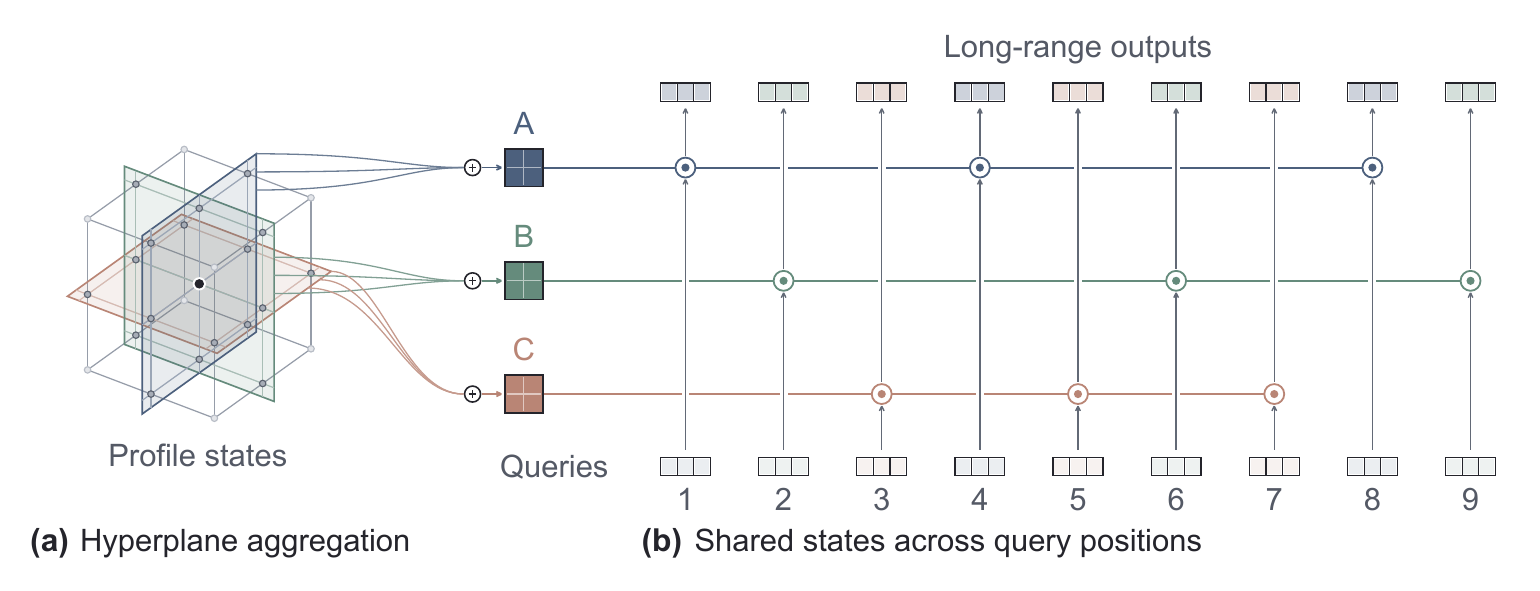}
    \caption{{Hyperplane aggregation and shared-state reads in SMat-Attention.}
    \textbf{(a)} A profile table is indexed by $\mathbb{F}_3^3$.
    Three coordinate hyperplanes each aggregate their
    incident profile states into a type state (A-C). The central
    profile belongs to all hyperplanes, so its state contributes to all
    the sums. \textbf{(b)} Each type state supplies every query of that type.
    The horizontal lines carry fixed shared
    states throughout these reads. The output scatters its results back to this order.  The figure shows augmented long-range
    contributions, before causal addition and final normalization.}
    \label{fig:placeholder}
\end{figure}
\textbf{The causal branch.}
Following standard chunkwise formulations of linear attention \citep{hua2022transformerqualitylineartime, beck2026tiled, yang2024gated}, partition $[n]$ and $(n,T]$ separately into $N_c$ chronological chunks
$I_1,\dots,I_{N_c}$ of width at most $c$; Set $n=c\lfloor T/2c\rfloor$, so no
chunk crosses block boundaries. Each chunk forms its summary $D_b=\sum_{j\in I_b}\mZ_j$, an exclusive prefix scan \citep{blelloch1990prefix, yau2026sequentialparalleldualityprefixscannable} over chunks gives the incoming state $S^{\mathrm{in}}_b=\sum_{a<b}D_a$, and all chunks then run in
parallel:
\begin{equation}
  \bar y^{\,\mathrm{loc}}_i
  =\underbrace{\phi_i^\top S^{\mathrm{in}}_b}_{\text{earlier chunks}}
  +\underbrace{\phi_i^\top\!\!\sum_{j\in I_b,\,j\le i}\!\!\mZ_j}_{\text{within chunk }b},
  \qquad
  \bar y_i=\bar y^{\,\mathrm{lr}}_i+\bar y^{\,\mathrm{loc}}_i,
  \qquad i\in I_b,
  \label{eq:local}
\end{equation}
followed by a row-wise division. 

\begin{algorithm}[hbt!]
\caption{SMat forward pass for one attention head}
\label{alg:prefill}
\begin{algorithmic}[1]
\REQUIRE $\mQ,\mK,\mV$; maps $\operatorname{prof},\operatorname{type}$;
  incidence $\mathbf C$; chunk width $c$; $n=c\lfloor T/2c\rfloor$
\STATE Compute $\phi_i$ and $\mZ_j$ as in Equation \ref{eq:augmented-kv-state}.
\STATE Pool $F_x=\sum_{\operatorname{prof}(j)=x}\mZ_j$ over $j\le n$ by grouped reduction.
\STATE Tabulate $U_h=\sum_x\mathbf C_{hx}F_x$ for every type $h\in\{0,\dots,B-1\}$.
\STATE \textbf{for all} types $h$ \textbf{in parallel}: contract all queries of type $h$ against $U_h$ and scatter the rows into chronological order.
\STATE \textbf{for all} chunks $b$ \textbf{in parallel}: $D_b=\sum_{j\in I_b}\mZ_j$.
\STATE Exclusive prefix scan over chunks, separately within $[n]$ and within $(n,T]$.
\STATE \textbf{for all} chunks $b$ \textbf{in parallel}: evaluate Equation \ref{eq:local}, add the long-range term, and normalize once.
\end{algorithmic}
\end{algorithm}

\textbf{Implicit incidence multiplication.} We compute the type table $U=\mC F$ without materializing $\mC$.
Let $\mathcal A$ contain the normalized nonzero directions in
$\mathbb F_q^{d-1}$, with the first nonzero coordinate of each $a$ equal to one.
Index each type by its defining pair $(a,b)$, using the same ordering as
$\operatorname{type}$.
For a fixed direction $a$, every point $x$ lies in exactly one hyperplane
$H_{a,b}$, namely the one with $b=a^\top x$.
Algorithm~\ref{alg:smat-incidence} therefore groups the profile states by this
offset and accumulates $U_{a,b}=\sum_{x:\,a^\top x=b}F_x$, where the offsets use finite-field arithmetic and the state additions use ordinary arithmetic in $\mathbb R^{r\times p}$.

\begin{algorithm}[hbt!]
\caption{Implicit incidence tabulation}
\label{alg:smat-incidence}
\begin{algorithmic}[1]
\REQUIRE Profile states $\{F_x\}$, field size $q$, dimension $d$ 
\STATE \textbf{if} {$d=1$} \textbf{then}  $U_0 \gets F_0$, and \textbf{return} $U$
\FOR{$a\in\mathcal A$ \textbf{in parallel}}
    \STATE $U_{a,b}\gets 0_{r\times p}$ for all $b\in\mathbb F_q$
    \FOR{$x\in\mathbb F_q^{d-1}$}
        \STATE $b\gets a^\top x$ \hspace{1.9cm} \textcolor{blue}{\texttt{/$\star$ Finite-field arithmetic}}
        \STATE $U_{a,b}\gets U_{a,b}+F_x$ \qquad \textcolor{blue}{\texttt{/$\star$ Ordinary state addition}}
    \ENDFOR
\ENDFOR
\STATE \textbf{return} $U$
\end{algorithmic}
\end{algorithm}

For $d\geq2$, the algorithm performs 
\begin{equation}|\mathcal A|q^{d-1}
    =\frac{q^{d-1}-1}{q-1}\,q^{d-1}
    =Bq^{d-2}
    =\Theta(q^{2d-3})
\end{equation}
state additions: one for each nonzero of $\mC$.
We enumerate points with a base-$q$ counter and maintain $a^\top x$ as its entries change. Over a full traversal, the counter changes $O(q^{d-1})$ entries, so generating the offsets takes $O(|\mathcal A|q^{d-1})$ field operations. So, the total work is $O(q^{2d-3}rp)$, and the profile and type tables occupy $O((q^{d-1}+B)rp)=O(q^{d-1}rp)$ words. For $q=\Theta(T^{1/d})$, these bounds become $O(T^{2-3/d}rp)$ work and $O(T^{1-1/d}rp)$ words. If $d=1$, we directly use the single profile state.\looseness=-1

\begin{wrapfigure}[13]{l}{0.40\textwidth}
    \centering
    \includegraphics[width=\linewidth]{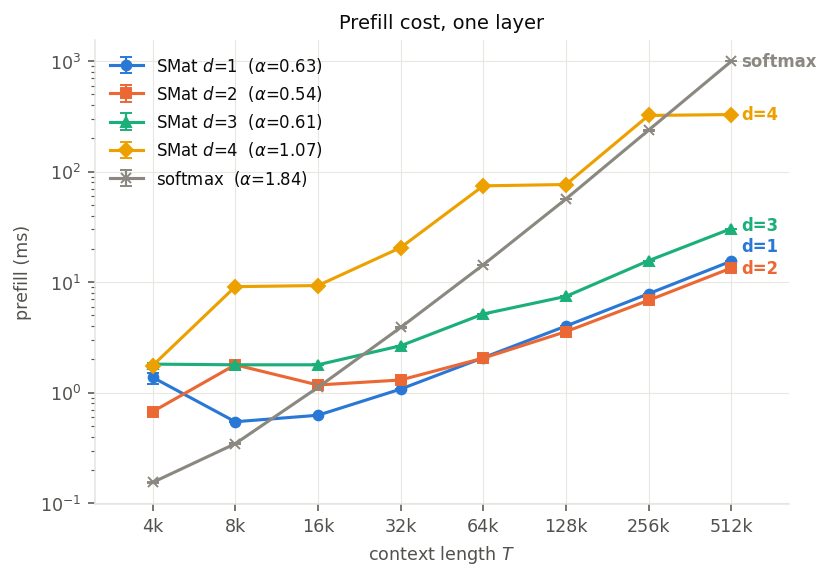}
    \caption{Prefill cost against fused causal attention (bf16, $8$ heads, $r=64$,
  chunk $128$; median of $15$ timed runs). }
\label{fig:prefill}
\end{wrapfigure}

\begin{restatable}[Chunkwise SMat-attention]{thm}{chunkwiseSmatAttention}
\label{thm:exact} Fix $d\geq 1$, and let $\mathbf M^{(d)}$ be the mask of
Section~\ref{subsec: mask constructions} under the prescribed sequence-length scaling. Assume the kernel has a
nonnegative feature factorization of dimension $r$, and every normalizer
$\bar y_{i,p}$ is positive. Then, for large $T$, Algorithm~\ref{alg:prefill} computes the
attention outputs $o_1,\dots,o_T$, in
\begin{equation}
  O\big((T^{2-3/d}+T)\,rp\;+\;T\,c\,(r+p)\big)
  \label{eq:cost}
\end{equation}
work, excluding evaluation of the feature maps, using
$O\big((T^{1-1/d}+T/c)rp+c^2\big)$ words of working memory. The backward pass
has the same asymptotic cost.
\end{restatable}
The proof of Theorem \ref{thm:exact} is in Appendix \ref{sec:chunking_proofs}.

\subsection{Gating}
Recent variants of linear attention, such as gated linear attention ~\citep{yang2024gated} and Mamba-2 ~\citep{dao2024transformers}, yield increased performance by weighting each edge according to the product of the gates between its source and target. To integrate our binary mask $\mathbf M^{(d)}$ into these continuous layers, we introduce a gating mechanism. Let $a_1, \ldots, a_T \in (0, 1]$ be per-token scalar decay gates, and $\lambda_i = \sum_{k \le i} \log a_k$. We replace the binary mask $\mathbf M^{(d)}$ with the continuous mask $\widetilde{\mathbf M}^{(d)}$, given by:
\begin{equation}
\label{eq:gated}
    \widetilde{\mathbf M}^{(d)}_{ij} = \mathbf M^{(d)}_{ij} \exp(\lambda_i - \lambda_j), \quad j \le i
\end{equation}
Note that the write scale of these layers is the key gate $h_j$ from \eqref{eq:augmented-kv-state}, which is already subsumed within $\mZ_j$, and that the entries of $\widetilde{\mathbf M}^{(d)}$ are non-negative with the same support as $\mathbf M^{(d)}$. We show this additional gating does not affect training or decoding complexity in Lemma \ref{thm:gated}. \looseness=-1
\subsection{Memory-efficient decoding}
\label{subsec:decoding}

For SMat-Attention with hard-routing, we compress the distant block once and reuse its summaries throughout the decoding process. For this, we begin by fixing the distance/recent boundary $n$ in advance. After processing the distant tokens, we tabulate the type states $\{U_h\}_{h<B}$ and discard the intermediate profile states $\{F_x\}$. At each subsequent step $t>n$, we update a single running state $S_t = S_{t-1} + Z_t$, initialized with $S_n=0$, and contract the query feature $\phi_t$ against $U_{\mathrm{type}}(t-n) + S_t$ before normalization. Figure \ref{fig:decoding-image} illustrates this separation between fixed distant memory and evolving recent memory. The resulting decoder retains $B+1$ states and performs $O(rp)$ state-update and readout work per token after tabulation. Appendix~\ref{app:window} relaxes the fixed boundary, extending the construction to decoding at any length without knowing $T$ in advance.\looseness=-1

\begin{figure}[hbt!]
    \centering\includegraphics[width=0.75\linewidth]{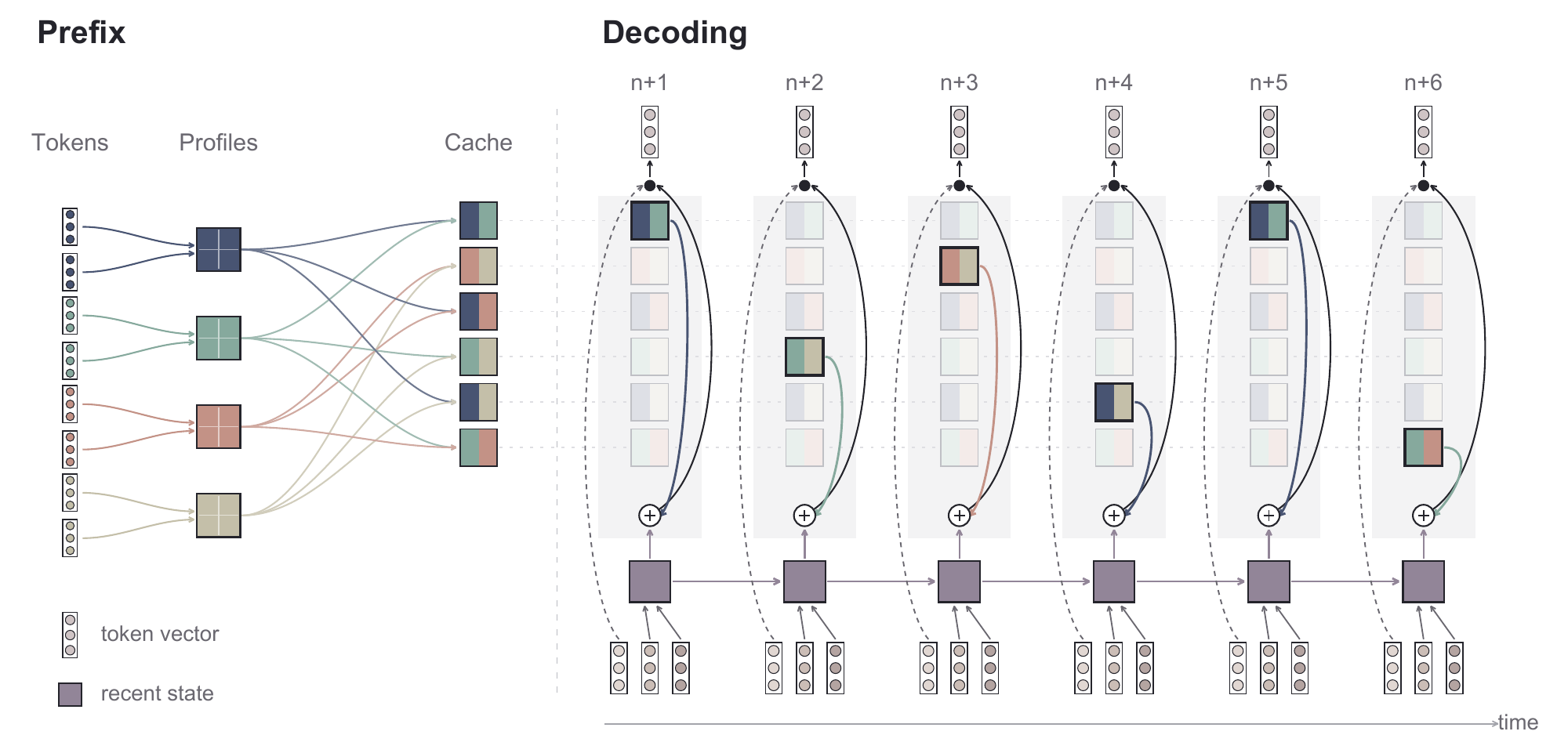}
    \caption{Cached decoding for additive SMat-Attention with hard-routing.
\textbf{Left:} distant key-value contributions pool into profile
states, which are added via the incidence structure to form the type cache. Colors identify the contributing profiles
within each cached summary. \textbf{Right:} successive decoding steps access the fixed cache while maintaining one running state for recent tokens. The
input vectors at each step are the query, key, and value. The key and value update the running state before answering the query. Repeated cache banks depict successive views of the shared memory.}
    \label{fig:decoding-image}
\end{figure}
\begin{restatable}[Streaming decoding]{thm}{streamingdecoding}
\label{thm:decode}
Assume the hypotheses of Theorem~\ref{thm:exact}, and suppose the total
sequence length $T$ is fixed in advance so that the distant/recent boundary $n$
is known. Once Algorithm~\ref{alg:prefill} has been run over the distant block,
each subsequent token can be decoded in $O(rp)$ time, independent of
$T$, from a cache of $O(T^{1-1/d}rp)$ words.
\end{restatable}
We provide the proof of Theorem \ref{thm:decode} in Appendix \ref{sec: decoding and horizon free decoding}.

\subsection{Recurrent architectures and learned routing}
 \label{sec:recurrent-extensions}

We extend SMat's profile organization to recurrent sequence
models. The local branch processes the distant and recent blocks
separately, resetting its recurrent state and short convolution at
the boundary. Distant tokens update profile memories $F_x$ (specified in Appendix~\ref{app:training}), which
are aggregated into type summaries $U_h=\sum_x C_{hx}F_x$.
For a recent query, the memory contribution is combined with the
local output before backbone normalization and output projection $
z_i=z_i^{\mathrm{local}}
+\lambda_i \widetilde q_i^\top\sum_h R_{ih}U_h$,
where $R_{ih}$ assigns queries to types and $\lambda_i$ controls the
memory contribution. To combine information across subsets, we introduce an
independent query-dependent scorer over types. Each query selects
the four highest-scoring types per head and applies a softmax over
their scores to obtain the read weights $R_{ih}$, retaining a
fixed number of summary reads.\looseness=-1

The learned selector remains subquadratic because each query scores
a sublinear number of cached types, $B=\Theta(T^{1-1/d})$, rather
than all tokens. Across the sequence, scoring costs
$\Theta(TB)=\Theta(T^{2-1/d})$, a factor of $T^{1/d}$ fewer scores
than token-level all-pairs scoring. Type-summary construction is
also subquadratic: Algorithm~\ref{alg:smat-incidence} costs $O(T^{2-3/d}rp)$, while a
dense contraction costs $O(T^{2-2/d}rp)$.
These routing and aggregation costs are subquadratic for every
fixed finite $d\ge2$ and fixed model dimensions. The linear-time bound for $d\le3$ applies to the hard-routing
algorithm; learned selection instead incurs the
subquadratic scoring cost above. 

For learned content routing, a token's write address is computed
from its current and preceding hidden states, while each recent
query selects summaries using only its own causal representation. A learned projection quantizes the write representation into a
discrete profile in $\mathbb{F}_q^{d-1}$. We train this hash jointly
with the backbone using a straight-through gradient estimator and an auxiliary load-balancing loss. The four-read selector is
learned separately through the softmax weights of its selected
types. Thus, the model learns both where to store information and
which summaries to retrieve; Appendix~\ref{app:content-routing}
gives the details. We show in Appendix \ref{sec: vc dimension proofs} that the VC-dimension of the support of the learned mask continues to be $O(d)$.\looseness=-1

\section{Experiments}
We evaluate SMat on controlled synthetic tasks designed to probe routing and retrieval, followed by long-context language modeling.

\textbf{Implementation and training details.} We implement our models in PyTorch with custom Triton kernels \citep{tillet2019triton}, and our experiments are run on single NVIDIA A100 GPUs.
We use the Zoology training pipeline for MQAR \citep{arora2023zoologymeasuringimprovingrecall}, and AdamW with cosine learning-rate decay. MQAR, joint context--key recall, and the 750M-token PG-19 experiments use the learned four-read routing variant of Section~\ref{sec:recurrent-extensions}. Subset routing and multi-key retrieval use the mask-based constructions. All SMat masks were gated in the recurrent section of the mask. We provide our code\footnote{\url{https://anonymous.4open.science/r/smat_attention/README.md}} and defer details to Appendix \ref{app:training}.

\textbf{Subset Routing.} We first test whether models using the proposed mask can realize the subset-selection patterns predicted by its VC dimension. Let $J=\{j_1,\ldots,j_k\}$ denote $k$ landmark positions in the context.
Each landmark $j_\ell$ stores an independent random payload $x_\ell$ in a
distinct output channel, $v_{j_\ell}=x_\ell e_\ell$, while all non-landmark values are zero.  A query specifies a subset $A\subseteq J$, and the target is $y_A = \sum_{j_\ell\in A} x_\ell e_\ell$. Payloads and requested subsets are resampled across examples, preventing the
model from memorizing fixed input-output mappings.

\begin{wrapfigure}[15]{r}{0.45\textwidth}
\vspace{-10pt}
    \centering    \includegraphics[width=\linewidth]{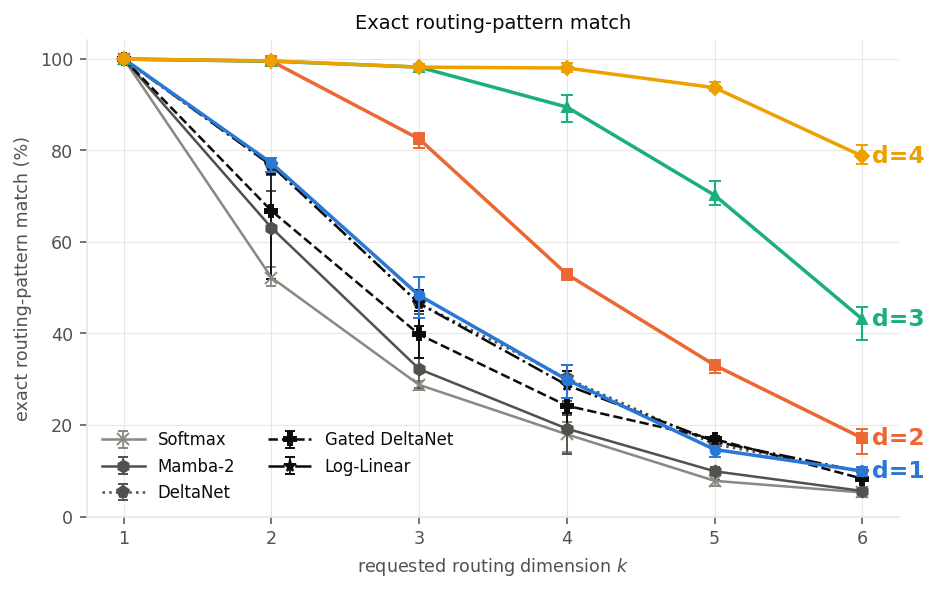}
    \caption{Routing-pattern match against the requested dimension k, after training with
the mask fixed, T = 1024, mean of three seeds. \looseness=-1}
\label{fig:routing-pattern}
\end{wrapfigure}

This task tests the access patterns characterized by VC dimension. If $k$ landmarks are shattered by the row supports of $\mM$ with VC-dimension $d$, then every subset of those landmarks can be selected by some query and $\mM$ can realize all subset-selection patterns over some $d$ landmarks, but not over any $d+1$ landmarks. So, we expect performance to degrade once the requested routing dimension exceeds $d$. The task captures a basic requirement of long-context retrieval: selecting several relevant pieces of information while ignoring other nearby or similar context.\looseness=-1

\textbf{Multi-key retrieval.}
We next test content-based retrieval when a single query must recover multiple items. We randomly place $N$ key-payload pairs throughout the distant context. Each query specifies $k$ target keys and must retrieve the corresponding payloads. Since the locations of the pairs vary across examples, the model cannot solve the task using fixed positional routing. For SMat, any set of at most $d-1$ target profiles in $\mathbb{F}_q^{d-1}$ lies on a common affine hyperplane. This guarantees that the target profiles lie on a common hyperplane, but does not guarantee exact selection: distractors or profile collisions may also lie on that hyperplane. We report exact-support accuracy (Appendix~\ref{app:experiments}), where a query is `correct' only if every requested item is present in the output and every unrequested item is absent.\looseness=-1

\begin{table}[H]
\scriptsize
\centering
\begin{tabular}{lrrr}
\hline
Model & $k=1$ & $k=2$ & $k=3$ \\
\hline
Softmax & 44.29 (49.00) & 66.60 (57.68) & 66.46 (57.55) \\
Linear Baselines & 0.00 (0.00) & 0.00 (0.00) & 0.00 (0.00) \\
LSH bucketing (125 buckets) & 100.00 (0.00) & 1.71 (0.50) & 0.03 (0.03) \\
\rowcolor{gray!12}
SMat ($d=2$) & 100.00 (0.00) & 13.90 (2.13) & 1.29 (0.33) \\
\rowcolor{gray!12}
SMat ($d=3$) & 100.00 (0.00) & 99.82 (0.10) & 34.20 (1.21) \\
\rowcolor{gray!12}
SMat ($d=4$) & 99.98 (0.03) & 99.92 (0.03) & 98.99 (0.03) \\
\hline
\end{tabular}
\caption{Exact routing-pattern match (\%) on multi-key retrieval, by the number $k$ of marked positions. Payloads are resampled each batch and evaluation uses fresh ones. Mean (std) over 3 seeds. SMat performs better at higher $k$ as $d$ increases.\looseness=-1}
\label{tab:mkar}
\end{table}

\textbf{Multi-query associative recall.}
We train on the Zoology MQAR \citep{arora2023zoologymeasuringimprovingrecall} mixture with 4-64 key-value
pairs at sequence lengths 64--256, using two-layer models
with head and state dimension 16.
Evaluation uses 1000
held-out examples per load at sequence lengths 64-256. Accuracy is averaged over query tokens within each example,
then equally over examples, giving equal weight to each load.
Log-Linear Attention \citep{guo2025loglinearattention} uses the same training harness and
budget as the corresponding backbone. Table \ref{tab:mqar-results} reports retrieval accuracy for Mamba-2 and Gated DeltaNet with and without SMat. At every width, for both backbones, the SMat variants perform well.\looseness=-1

\begin{table}[h]
\centering
\scriptsize \caption{Final MQAR accuracy (\%), 32 epochs. Mean (std) over 3 seeds.}
\label{tab:mqar-results}
\begin{tabular}{
    l
    r
    r
    r
    >{\columncolor{gray!12}}r
    >{\columncolor{gray!12}}r
    >{\columncolor{gray!12}}r
}
\toprule
Backbone & Width & Base & Log-Linear & $d=2$ & $d=3$ & $d=4$ \\
\midrule
Gated DeltaNet & 16 & 44.34 (7.58) & 44.17 (10.90) & 49.82 (5.21) & 45.40 (3.70) & 44.63 (1.39) \\
Gated DeltaNet & 32 & 63.99 (4.14) & 68.41 (8.13) & 70.07 (0.48) & 79.74 (2.65) & 75.21 (6.78) \\
Gated DeltaNet & 64 & 70.59 (7.96) & 79.06 (2.11) & 78.88 (8.01) & 87.37 (4.12) & 80.52 (6.07) \\
\midrule
Mamba-2 & 16 & 41.36 (3.07) & 46.47 (7.61) & 47.61 (4.36) & 61.57 (2.41) & 49.19 (5.93) \\
Mamba-2 & 32 & 73.55 (5.02) & 76.57 (4.90) & 75.44 (0.19) & 77.89 (4.45) & 81.90 (2.49) \\
Mamba-2 & 64 & 87.87 (4.26) & 85.85 (1.92) & 91.76 (0.95) & 92.89 (2.13) & 93.14 (1.55) \\
\bottomrule
\end{tabular}
\end{table}

\textbf{Joint Context--Key Recall}
We adapt the multi-query joint recall task of ~\cite{zhan2025jointrecall}, which requires retrieving
values using both a context and a key. We represent 
each binding as an explicit $(\text{context}, \text{key}, \text{value})$
record and independently shuffle the records and queries. Keys repeat across contexts, while values are sampled independently, requiring joint identification of the requested record. We train on 180K examples spanning 4, 16, 128, 256, and 512 bindings, using two-layer, width-64 models for 32 epochs. Table~\ref{tab:joint_recall} reports validation accuracy averaged across the five memory loads. Among the tested SMat settings, $d=3$ achieves the highest
mean accuracy and lowest sample standard deviation on both
backbones. It outperforms the native backbones in mean accuracy,
while scoring above Log-Linear on GDN and below it on Mamba-2.
These results do not establish the cause of the differences
across $d$.

\begin{table}[H]
\small
\centering
\begin{tabular}{lrr}
\hline
Variant & Mamba-2 & GDN \\
\hline
Native & 51.14 (17.74) & 53.78 (2.85) \\
\rowcolor{gray!12}
+ Log-Linear & 66.01 (19.88) & 57.78 (5.36) \\
\rowcolor{gray!12}
+ SMat ($d=2$) & 54.10 (3.72) & 51.25 (5.28) \\
\rowcolor{gray!12}
+ SMat ($d=3$) & 61.65 (3.15) & 60.15 (3.27) \\
\rowcolor{gray!12}
+ SMat ($d=4$) & 59.12 (15.12) & 59.89 (9.34) \\
\hline
\end{tabular}
\caption{Average validation accuracy (\%) on joint context--key recall, averaged over five memory loads. All models use width 64, 32 epochs, learning rate 0.003. Mean (std) over 3 seeds.}
\label{tab:joint_recall}
\end{table}

\textbf{Natural Language Modeling.}
We use language modeling to evaluate whether the additional routing structure from SMat preserves the modeling capabilities of the underlying architectures. We evaluate language modeling on PG-19 \citep{rae2019compressivetransformerslongrangesequence}
using the GPT-2 tokenizer. We train separate models at context lengths of 16K and 32K, each on 300M tokens, and report per-token loss on held-out tests at the corresponding training context length. All models use 8 layers, a hidden width of 384, and the same training data order. SMat, Mamba-2, and the Transformer have 27.8M, 26.9M, and 33.5M parameters, respectively. Across both context lengths, SMat variants achieve slightly lower negative log-likelihood (NLL) than the similarly-sized Mamba-2 baseline, with small differences among \(d=2,3,4\), as seen in Table \ref{tab:pg19-300m} in Appendix ~\ref{app:training}.\looseness=-1

We further evaluate SMat augmentation of GDN and Mamba-2 at a larger scale, using 16-layer models with hidden width 768, a 16K context length, and 750M training tokens. As shown in \ref{tab:pg19-750m}, all SMat variants achieve roughly the same perplexity as its baselines, suggesting that SMat's performance in language-modeling is comparable across both backbone architectures in this setting. 

\section{Conclusion and Future Work}

We introduced SMat-Attention, a framework for explicitly trading off the flexibility of long-range access against computation and memory. Rather than compressing the entire past into a fixed-size state or allowing unrestricted token-level interactions, SMat-Attention provides an intermediate regime in which the richness of long-range routing is controlled by a single parameter, the VC dimension \(d\). This structure yields provably subquadratic training and constant-time decoding for the hard-routing masks, and the learned extensions inherit those bounds up to a subquadratic selector term. Controlled routing experiments show benefits consistent with this expressiveness. Learned SMat extensions improve mean associative and joint context–key recall accuracy in several tested settings, while remaining comparable in performance at small-scale language modeling.\looseness=-1

We focus on a structured finite-feature setting and a fixed-horizon decoding formulation, while some empirical variants introduce additional learned routing and memory updates. Extending the framework to more adaptive routing schemes, dynamic contexts, and larger-scale language models is a natural direction for future work. Our results suggest that explicitly controlling the complexity of long-range access patterns is a useful way to navigate tradeoffs between expressivity and computation.\looseness=-1

\section*{AI use statement}

In this work, we used generative AI tools for literature searches, coding implementation, and to aid in the presentation of our experimental results. We also used AI assistance to edit the manuscript, help design scientific figures, and validate our mathematical claims. The authors take full responsibility for verifying the correctness and originality of all the material in the manuscript, including the theoretical claims, experimental
results, code, and figures.

\section*{Ethics Statement}

In this work, we investigate the computational properties of the attention mechanism used in the transformer, and study tradeoffs between the memory, expressiveness, and computation across various representations. Our evaluation uses synthetic tasks and Google DeepMind's PG-19 corpus for measuring the ability of language models to process long-range contexts. In general, efficiency gains in this line of work may broaden the access to long-context modeling, while also lowering the cost of processing potentially sensitive pieces of textual information. Moreover, models using the proposed mechanism are subject to the standard bias, privacy, and misuse risks associated with language models, and the attention mechanism on its own does not provide safeguards against these risks. We discuss broader impacts in Appendix \ref{sec: broader impacts}.

\section*{Reproducibility Statement}

Section \ref{subsec: mask constructions} specifies the mask construction and attention algorithms, including pseudocode for the forward pass and implicit incidence tabulation. We provide proofs of our theoretical claims in Appendices \ref{sec: miscellaneous proofs}-\ref{sec: decoding and horizon free decoding}. Section \ref{app:training} describes the experimental tasks, datasets and synthetic-data generation procedures,
model configurations, optimization settings, training budgets, evaluation
protocols, and computing hardware. Finally, we provide an anonymous link to a  faithful implementation of our code in the main body of our paper, along with documentation on how to run the experiments. 

\section{Acknowledgements}  We are deeply grateful to Jan van den Brand, Jacob Abernethy, Peter Bartlett, Sarah Liaw, Avi Feller, and Ali Behrouz for sharing their helpful ideas and insightful discussions. Emile Anand is supported by NSF Grant CCF 2338816. Abdullah Ateyeh and Archer Wang are supported by the NSF graduate research fellowship. This research was also sponsored by the Department of the Air Force Artificial Intelligence Accelerator and was accomplished under Cooperative Agreement Number FA8750-19-2-1000. The views and conclusions contained in this document are those of the authors and should not be interpreted as representing the official policies, either expressed or implied, of the Department of the Air Force or the U.S. Government. The U.S. Government is authorized to reproduce and distribute reprints for Government purposes notwithstanding any copyright notation herein. In addition, this work is supported by the National Science Foundation under Cooperative Agreement PHY-2019786 (The NSF AI Institute for Artificial Intelligence and Fundamental Interactions, {http://iaifi.org/}).  

\bibliography{iclr2027_conference}
\bibliographystyle{plainnat}

\newpage

\appendix

\textbf{Outline of the Appendices}.
\begin{itemize}
\item Section \ref{sec: broader impacts} explains the broader societal impacts of our work,
\item Section \ref{sec: related work} lists the related work,
\item Section \ref{sec: miscellaneous proofs} states an auxiliary lemma to motivate our kernel attention mechanism,
\item Section \ref{sec: vc dimension proofs} proves our theorem to characterize the properties of the masks,
\item Section  \ref{sec:chunking_proofs} proves the chunkwise SMat-attention result, 
\item Section \ref{sec: decoding and horizon free decoding} discusses extension to horizon-free decoding, and
\item Section \ref{app:training} lists the training details
\end{itemize}

\section{Broader Impacts}
\label{sec: broader impacts} 
SMat-Attention targets the computational and memory costs of long-context sequence modeling. Namely, more efficient prefill and decoding could lower the energy and hardware requirements of deploying long-context models, potentially widening access to researchers and practitioners without large compute budgets. By making long-range access complexity an explicit architectural parameter, our framework also offers a more interpretable theory on which parts of the context a model can attend to, which may aid analysis of how long-context models retrieve and use information. At the same time, cheaper long-context inference lowers the barrier to processing large volumes of personal or sensitive text, and the general risks of language models, including the generation of misleading or harmful content, apply to systems built on this mechanism.

 \section{Related Work}
\label{sec: related work}
\begin{table}[h]
\centering
\begin{tabular}{lccc}
\toprule
& prefill & per decoded token & decode cache \\
\midrule
Softmax attention   & $\Theta(T^2 d_v)$ & $\Theta(Td_v)$ & $\Theta(T)$ pairs \\
Linear attention    & $\Theta(Trp)$     & $\Theta(rp)$   & $1$ state \\
Log-linear attention & $\Theta(T\log T\,rp)$  & $\Theta(rp\log T)$  & $\Theta(\log T)$ states \\
SMat, VC dim.\ $d$  & $O((T^{2-3/d}+T)rp)$ & $O(rp)$ & $O(T^{1-1/d})$ states \\
\bottomrule
\end{tabular}
\caption{SMat sits between linear and softmax attention. It keeps the
$O(rp)$ per-token decoding cost of linear attention and pays for VC dimension
$d$ in cache size rather than in decoding time.}
\label{tab:cost}
\end{table}

\textbf{Kernel and recurrent attention.} Kernelized attention factorizes the content kernel as
$K(q,k)=\phi_\mQ(q)^\top\phi_\mK(k)$, allowing causal attention to be accumulated in finite-dimensional recurrent states \citep{kacham2023polysketchformer}. For fixed feature and value dimensions, this yields linear work in the sequence length, as well as a recurrent state whose size is independent of $T$ \citep{katharopoulos2020transformersrnnsfastautoregressive}. Random-feature methods such as Performer \citep{choromanski2022rethinkingattentionperformers} approximate the softmax kernel within this framework. More recent architectures enrich the recurrent state and its update \citep{sun2023retentivenetworksuccessortransformer, katsch2024gateloopfullydatacontrolledlinear, qin2024hgrn2gatedlinearrnns, peng2024eaglefinchrwkvmatrixvalued}: for instance, gated linear attention introduces input-dependent retention \citep{peng2021randomfeatureattention, yang2024gated}, whereas DeltaNet uses key-conditioned corrective updates \citep{schlag2021linear, yang2024parallelizing}. These mechanisms improve how a fixed-size state is maintained, but retaining a sequence-length-independent state creates a capacity-recall tradeoff on tasks requiring access to many independent items from the context \citep{arora2025simplelinearattentionlanguage}. Our work differs: SMat-Attention accepts any supplied finite nonnegative feature factorization and changes the \emph{causal support pattern} over which the resulting features interact.

\textbf{Structured sequence mixing and hierarchical memory.} A recurring theme in efficient sequence modeling is that computational savings arise from algebraic structure in the sequence mixing matrix since many efficient sequence models can be interpreted as multiplication by a structured causal matrix. For instance, linear attention induces a lower-triangular structured operator \citep{katharopoulos2020transformersrnnsfastautoregressive}, whereas long-convolution architectures such as Hyena \citep{massaroli2023laughinghyenadistilleryextracting,poli2023hyena} use Toeplitz-like operators \citep{qin2023toeplitzneuralnetworksequence,feinashley2025spectrefftbasedefficientdropin} with FFT-based multiplication.
 Similarly, selective state-space models (SSMs) such as Mamba induce input-dependent semi-separable mixing matrices \citep{gu2023mamba}. Mamba-2 makes this matrix perspective explicit through structured state-space duality, showing an equivalence between SSM recurrences and multiplication by semi-separable matrices \citep{dao2024transformers}. More generally, \citet{choromanski2022block} showed that efficient multiplication by a mask can be lifted to efficient finite-feature masked attention, encompassing causal prefix masks, relative-position operators, and a variety of graph-derived masks. These examples suggest treating the structure of the sequence mixing matrix
itself as a design space. 

 The most closely related work to ours is Log-Linear Attention, which replaces linear attention's single prefix state with $O(\log T)$ summaries of disjoint dyadic buckets maintained through a Fenwick-tree schedule \citep{guo2025loglinearattention}. Its hierarchical matrix structure supports $O(T\log T)$ parallel training and $O(\log T)$ time and memory per decoded token, while query-dependent coefficients select among temporal scales. We explore a different structural axis in SMat-Attention by constructing overlapping binary access patterns from point-hyperplane incidences, and quantifying their combinatorial richness via the VC-dimension. Therefore, SMat Attention's states summarize geometric profiles rather than temporal buckets, and our resulting guarantee is different: we show that after a fixed, known distant prefix has been processed, SMat-Attention decodes each subsequent token in time independent of $T$ using $O(T^{1-1/d})$ cached feature states.
 
\textbf{VC dimension and structured matrix multiplication.}  Beyond sequence modeling, a line of work studies when structured
matrices result in fast matrix-vector multiplication. The VC-dimension is a combinatorial complexity metric that classically measures the richness of a set system through the subsets it realizes \citep{kearns1994introduction}. Recent work connects this quantity to the complexity of matrix-vector multiplication. For instance, after an $\tilde{O}(T^2)$ preprocessing of a $T\times T$ Boolean matrix $\mathbf M$ of VC-dimension $d$, \citet{anand2025structuralcomplexitymatrixvectormultiplication} gives an algorithm for multiplying $\mathbf M$ by an arbitrary vector in $\widetilde O(T^{2-1/d})$ time. This connection motivates our use of VC dimension as an access-complexity measure; our construction additionally exploits point–hyperplane incidence structure to obtain a sharper specialized algorithm.

\textbf{Sparse and hardware-efficient attention.} Hardware-aware algorithms such as FlashAttention reorganize exact softmax attention into on-chip tiles, substantially reducing memory traffic without changing its worst-case quadratic arithmetic complexity \citep{dao2022flashattentionfastmemoryefficientexact}. Sparse-attention methods instead reduce the number of evaluated query-key pairs through local and global windows, random edges, or content-dependent selection \citep{beltagy2020longformerlongdocumenttransformer, 3495724.3497174, yuan2025native,li2019enhancing}. Conversely, SMat-Attention's complete causal mask has $\Theta(T^2)$ non-zeros, and therefore cannot be evaluated efficiently by enumerating all permitted interactions. Although its long-range block contains $\Theta(T^{2-1/d})$ token-level edges, sparsity alone would only yield the corresponding $\Theta(T^{2-1/d})$ computation. Our sharper bound comes from additional reuse: distant keys with the same profile are pooled once, queries with the same type share an aggregate, and the two causal blocks are handled by prefix scans and local dense tiles. Thus, SMat-Attention combines hardware-friendly intrachunk computation with algebraic reuse across chunks; its speedup is not merely a consequence of deleting attention edges. Additionally, the Reformer's LSH attention \citep{kitaev2020reformerefficienttransformer} and the Routing Transformer \citep{roy2020efficientcontentbasedsparseattention} restrict each query to the keys sharing its hash or cluster, and our content-based assignment uses a similar device. The structures differ for $d \ge 3$: hash and cluster buckets partition the keys, so a query sees exactly one cell, whereas each of our hyperplanes contains $q^{d-2}$ profiles, so a query reads a structured union of cells. For $d = 2$ the hyperplanes are single points and SMat reduces to bucketed linear attention.\looseness=-1

\textbf{Linear attention and its variants.}
Linear attention replaces the softmax kernel with a feature map that factorizes the attention matrix, allowing $
O_t = \phi(q_t)^\top \sum_{j\le t}\phi(k_j)v_j^\top $
to be computed using a fixed-size recurrent state. This reduces autoregressive decoding from linear to constant cost per token and allows for efficient parallel training, but compresses the entire history into a fixed-size state \citep{katharopoulos2020transformersrnnsfastautoregressive}. However, vanilla linear attention compresses history into a fixed-size state by accumulating key--value associations, leading to memory-capacity limitations and interference between stored associations \citep{schlag2021linear}. Gated variants augment this recurrence with input-dependent retention factors that modulate the existing state, allowing the model to selectively preserve or forget past information while retaining efficient recurrent inference. For example, gated linear attention updates the state as $S_t = A_t \odot S_{t-1} + v_t k_t^\top$, where the gate $A_t$ determines which parts of the previous state are retained, while preserving their efficient recurrent and parallel forms \citep{yang2024gated}. Gating improves memory management by controlling how much of the existing state is retained, but it does not directly account for what value is already stored at a particular key. Delta-rule models make the update key-specific by using the current prediction error,
$S_t = S_{t-1}(I-\beta_t k_tk_t^\top)+\beta_t v_tk_t^\top$, so that writing at $k_t$ partially removes the value currently associated with that key before inserting $v_t$ \citep{yang2024parallelizing}.

\section{Auxiliary Lemmas}
\label{sec: miscellaneous proofs}

\begin{lemma}
    Assume that the mask $\mM \in \R^{T\times T}$ supports matrix-vector multiplication in time $f_\mM(T)$. Then, the general masked kernel attention algorithm with mask $\mM$ can be implemented in time $O((f_\mM(T) + T)rd_v)$.
\end{lemma}
\begin{proof}
    Note that the $i$'th token representation obtained from the general masked kernel attention has the form 
    \[\mathsf{Att}_K(\mQ, \mK, \mV,\mM)_i= \frac{\phi(\mathbf q_i^\top)^\top \sum_{j=1}^T \mM_{i,j} \phi(\mathbf k_j^\top)^\top v_j}{\phi(\mathbf q_i^\top)^\top \sum_{j=1}^T \mM_{i,j}\phi(\mathbf k_j^\top)}.\]
Then, following \cite{luo2021stable}, let \[\mathbf H^1 = \left(\sum_{j=1}^T \mM_{i,j} \phi(\mathbf k_j^\top) \mathbf v_j\right)_{i=1}^T \in \R^{r\times d_v}\] and \[\mathbf H^2 = \left(\sum_{j=1}^T \mM_{i,j}\phi(\mathbf k_j^\top)^\top\right)_{i=1}^T \in \R^{1\times r}.\] Note that if $\mathbf D^1$ and $\mD^2$ are the vectorized forms of $\mH^1$ and $\mH^2$ (respectively), where each element of the sequence is vectorized and the resulting vectors are stacked into matrices, then $\mD^1 = \mM \mV^1$ and $\mD^2 = \mM \mV^2$, where the $i$'th rows of $\mV^1$ and $\mV^2$ are given as $\mV^1_i = \mathrm{vec}(\phi(\mathbf k_i)^\top \mathbf v_i)$ and $\mV_i^2 = \phi(\mathbf k_i^\top)^\top$. Therefore, computing $\mH^1$ and $\mH^2$ takes time $f_\mM(T)rd_v$, and so $\mathsf{Att}_i$ can be computed in time $O((f_\mM(T) + T)rd_v)$, which completes the proof.\qedhere\\
\end{proof}

\section{VC dimension of the causal incidence masks}
\label{sec: vc dimension proofs}
\begin{lemma}[Lower triangular matrices]
    The $T\times T$ binary lower triangular matrix $\mL_T$ has VC-dimension $0$ for $T=1$ and $1$ for $T\geq 2$.\label{lemma: L_T has vc dimension 1}
\end{lemma}
\begin{proof}
For $T=1$, the only row support is $\{1\}$, so no singleton is shattered.

For $T\ge2$, note that column $2$ is excluded by row $1$ and included by
row $2$, and so the singleton $\{2\}$ is shattered, which proves $\operatorname{VC}(\mL_T)\ge1$.
Next, take any two distinct column/row indices $x, y$, with $x<y$. In $\mL_T$, the upper-right entries are all $0$'s. So, if a row indicator functions is $1$ at a later index, we cannot have independent labelings. Therefore, since no configuration of two points can achieve all $2^2=4$ binary label combinations, the VC-dimension is less than $2$, completing the proof.\qedhere\\
\end{proof} 

\begin{lemma}[Affine hyperplanes]
\label{lem:affine-hyperplane-vc}
Let $D\ge1$ and let $q$ be a prime power. Then, the set system of all proper affine hyperplanes in $\mathbb F_q^D$ has VC dimension exactly $D$.
\end{lemma}
\begin{proof}
We first record an affine-dependence observation. Suppose profiles
$x_j\in\mathbb F_q^D$, indexed by a set $C$, satisfy
\[
\sum_{j\in C}\lambda_jx_j=0,
\qquad
\sum_{j\in C}\lambda_j=0,
\qquad
\lambda_j\ne0\quad(j\in C).
\tag{*}
\]
For every $a\in C$, any affine hyperplane containing all $x_j$ with
$j\in C\setminus\{a\}$ also contains $x_a$: indeed,
\[
x_a=-\sum_{j\in C\setminus\{a\}}
\frac{\lambda_j}{\lambda_a}x_j,
\qquad
-\sum_{j\in C\setminus\{a\}}\frac{\lambda_j}{\lambda_a}=1.
\]
This observation also applies to repeated profiles at distinct indices.

If $D+1$ points were shattered, the all-included trace would place
them in a proper hyperplane of affine dimension $D-1$. They are
therefore affinely dependent. Restricting a nonzero dependence to its
nonzero coefficients gives a set $C$ satisfying $(*)$. The preceding
observation rules out the trace $C\setminus\{a\}$, contradicting
shattering. This proves the upper bound.

For the lower bound, let $e_1,\ldots,e_D$ be the standard basis.
For each $\mR\subseteq[D]$, define
\begin{equation}
H_R=
\begin{cases}
\displaystyle\left\{x\in\mathbb F_q^D:
\sum_{\ell\notin R}x_\ell=0\right\}, & R\ne[D],\\[4pt]
\displaystyle\left\{x\in\mathbb F_q^D:
\sum_{\ell=1}^D x_\ell=1\right\}, & R=[D].
\end{cases}
\label{eq:vc-witness-hyperplanes}
\end{equation}
Each defining normal is nonzero, so every $H_R$ is a proper affine
hyperplane. Moreover, $e_\ell\in H_R$ if and only if $\ell\in R$.
Thus the standard basis is shattered. The construction is valid
over every finite field, including $\mathbb F_2$, which completes the proof.\qedhere\\
\end{proof}

\propertiesofsmatmasks*

\begin{proof}Let $\mM^{(d)}$ be the mask constructed in Section \ref{subsec: mask constructions}. Let $q$ be a prime. Then, for every $1\leq d < \lfloor \log_2 T\rfloor$ for sufficiently large $T$, we prove the above properties.

Recall the block form from Section \ref{subsec: mask constructions}:
\[
\mM^{(d)}=
\begin{pmatrix}
\mL_n&0\\
\mG&\mL_m
\end{pmatrix},
\qquad
\mG_{ij}=\mathbf1\{\operatorname{prof}(j)
\in H_{\operatorname{type}(i)}\}.
\]
The two diagonal blocks are inclusive causal triangles,
and every entry of $\mG$ connects a recent query to a distant
key. This proves causality, the unit diagonal, and
$\mM^{(d)}\le \mL_T$.

    We prove (iii) and (iv) for every $d$. First, for $d=1$, the mask is $\mL_T$ which clearly satisfies (i) and (ii) and (iv). Then, from \cref{lemma: L_T has vc dimension 1}, (iii) is satisfied.  Moreover, note that
\[
\mathrm{nnz}(\mM^{(d)})
\ge\frac{n(n+1)+m(m+1)}2
\ge\frac{(n+m)^2}{4}
=\frac{T^2}{4},
\]
while causality gives
$\operatorname{nnz}(\mM^{(d)})\le T(T+1)/2$.

Finally, it remains to prove the VC-dimension claim for $d\ge2$. For this, let $D=d-1$.

\textbf{Upper bound.} Suppose a set $J$ of $d+1$ column indices is shattered. Split it into distant and recent indices,
$J=J_{\mathrm{dist}}\sqcup J_{\mathrm{rec}}$.
The traces on the recent block are empty for distant rows
and prefixes for recent rows. In particular, if $u<v$
are recent indices, every row containing $v$ also contains
$u$. Thus $|J_{\mathrm{rec}}|\le1$. We proceed by casework:

If $|J_{\mathrm{rec}}|=1$, consider the labelings in which
that recent coordinate is one. Their realizing rows must
be recent rows, whose distant supports are $\{j\in[n]:\operatorname{prof}(j)
\in H_{\operatorname{type}(i)}\}$. Therefore, the $d=D+1$ distant indices would be shattered by proper affine hyperplanes in $\mathbb F_q^D$. If two of these indices share a profile, they cannot be independently labeled. Otherwise, their profiles would
form a shattered set of $D+1$ points, contradicting
Lemma \ref{lem:affine-hyperplane-vc}. Hence this case is impossible.

On the other hand, if $|J_{\mathrm{rec}}|=0$, all $d+1=D+2$ indices are
distant. Their profiles are affinely dependent, so there
exist a set $C\subseteq J$ and coefficients satisfying
\[
\sum_{j\in C}\lambda_j\operatorname{prof}(j)=0,
\qquad
\sum_{j\in C}\lambda_j=0,
\qquad
\lambda_j\ne0\quad(j\in C),
\]
where $|C|\ge2$. By the affine-dependence observation
in Lemma \ref{lem:affine-hyperplane-vc}, for every $a\in C$, any affine hyperplane containing the profiles indexed by $C\setminus\{a\}$ also contains $\operatorname{prof}(a)$.
So, no recent row realizes the trace
$C\setminus\{a\}$ on $C$.

Since shattering $J$ implies shattering $C$, all the
traces $C\setminus\{a\}$ would have to be realized
by distant rows. But distant row supports are nested
prefixes, whereas the $|C|\ge2$ sets $\bigl\{C\setminus\{a\}:a\in C\bigr\}$ are pairwise incomparable, and a chain cannot realize all
of them. Therefore, this contradiction proves $\operatorname{VC}(\mM^{(d)})\le d$.

\textbf{Lower bound.}
We now use the prescribed assignment conditions from
Section \ref{subsec: mask constructions}. These provide distinct distant positions
$j_1,\ldots,j_D$ with $\operatorname{prof}(j_\ell)=e_\ell$ for $\ell\in[D]$, and a recent index $\tau\in[m]$ such that each witness
hyperplane $H_R$, $\mR\subseteq[D]$, occurs at recent
query indices $i_R^-,i_R^+$ satisfying
\[
i_R^-<\tau\le i_R^+,
\qquad
H_{\operatorname{type}(i_R^-)}
=H_{\operatorname{type}(i_R^+)}=H_R.
\]
By the definition of these witness hyperplanes,
$e_\ell\in H_R$ if and only if $\ell\in R$.

Consider the $d$ column indices $J_*=\{j_1,\ldots,j_D,n+\tau\}$. Fix $A\subseteq J_*$ and set
$\mR=\{\ell:j_\ell\in A\}$.
If $n+\tau\notin A$, use row $n+i_R^-$;
if $n+\tau\in A$, use row $n+i_R^+$.
In both cases, the incidence condition gives exactly
the required trace on the distant landmarks.
The recent causal triangle includes column $n+\tau$
precisely when the recent query index is at least $\tau$,
so the final coordinate also has its required label.
Thus the selected row has trace exactly $A$ on $J_*$.

Therefore, every subset of $J_*$ is realized, giving
$\operatorname{VC}(\mM^{(d)})\ge d$.
Together with the upper bound, this proves
$\operatorname{VC}(\mM^{(d)})=d$.\qedhere
\end{proof}

\textbf{VC dimension of learned multi-read routing.}
Fix a sequence and head, and let $D=d-1\ge1$.
With one profile per distant key and $k$ positively
weighted reads, each cross-block support is a union
of $k$ affine-hyperplane traces. Writing $B_k$ for
this binary cross-block mask and
\[
  \widehat M_k=
  \begin{pmatrix}
    L_n&0\\ B_k&L_m
  \end{pmatrix},
\]
we have, for an absolute constant $C$,
\[
  \operatorname{VC}(\widehat M_k)
  \le C(d-1)k\log_2(2k)+2.
\]
In particular,
$\operatorname{VC}(\widehat M_4)\le25d-24=O(d)$.

\begin{proof}
Affine hyperplanes in $\mathbb{F}_q^D$ have VC
dimension $D$. By Sauer-Shelah's lemma, the number of
cross-block traces on any $s\ge D$ keys satisfies
\[
  \Pi_{B_k}(s)
  \le
  \left(\sum_{r=0}^{D}\binom{s}{r}\right)^k
  \le
  \left(\frac{es}{D}\right)^{Dk}.
\]
Thus, shattering requires $2^s\le(es/D)^{Dk}$.
Setting $u=s/(Dk)$ gives $2^u\le eku$, hence
$u=O(\log_2(2k))$ and
\[
  \operatorname{VC}(B_k)
  =O\!\left(Dk\log_2(2k)\right).
\]
For $k=4$, setting $u=s/D$ instead yields
$2^u\le(eu)^4$, which implies $u<25$.
Consequently, $\operatorname{VC}(B_4)\le25D-1$.

Finally, let $v=\operatorname{VC}(B_k)$.
A shattered column set contains at most one recent
column, since recent-column supports are nested
prefixes. If it contains one, fixing that column
to one forces recent rows to shatter all selected
distant columns, giving size at most $v+1$.
If all columns are distant, fixing the earliest
to zero and the latest to one excludes every
distant prefix row, so recent rows must shatter
the remaining columns. Thus
$\operatorname{VC}(\widehat M_k)\le v+2$,
proving both claims.
\end{proof}

\textbf{Dependence on the construction.} The bounds above characterize our family under its prescribed field-size scaling. Other choices of $q$ can preserve the same mask VC dimension while changing the number of profile cells, the number of tokens sharing each profile, and the computational cost. In particular, fixed $q$ gives a fixed number of profiles as $T$ grows. Our scaling instead lets the profile and type tables grow with sequence length.

\section{Chunking Proofs}
\label{sec:chunking_proofs}

We now provide the proof for our result in Theorem \ref{thm:exact}.

\chunkwiseSmatAttention*

\begin{proof}
\textbf{Forward Pass.} Before the final division, $\bar y_i$ is linear in the mask, and the two sub-masks of $\mM^{(d)}$ have disjoint support. Thus, it suffices to evaluate each branch separately, add the outputs, and then normalize.

For the long-range branch, a recent token $n+i$ receives
$\phi_{n+i}^\top\sum_{j\le n}\mathbf G_{ij}\mZ_j$. Substituting
$\mathbf G_{ij}=\mathbf C_{\operatorname{type}(i),\operatorname{prof}(j)}$ and
grouping the sum by profile,
\[
  \sum_{j\le n}\mathbf G_{ij}\mZ_j
  =\sum_{x\in\mathbb F_q^{d-1}}\mathbf C_{\operatorname{type}(i),x}
   \!\!\!\sum_{j\le n:\,\operatorname{prof}(j)=x}\!\!\!\mZ_j
  =\sum_x\mathbf C_{\operatorname{type}(i),x}F_x
  =U_{\operatorname{type}(i)},
\]
which is the entry of the type table computed by the algorithm. The
regrouping is valid because $\operatorname{prof}$ partitions
$[n]$ into disjoint sets. Furthermore, because this long-range branch only applies to recent queries, the term evaluates to zero for all distant queries
$\bar y^{\,\mathrm{lr}}_i=0$ for $i\le n$.

To be more explicit, each augmented state flattens into a row, so $\mZ\in\R^{n\times rp}$. Recall the one-hot matrices of \eqref{eq:C}: $\mE\in\{0,1\}^{n\times q^{d-1}}$ with $\mE_{jx}=\mathbbm 1\{\operatorname{prof}(j)=x\}$, and $\mR\in\{0,1\}^{m\times B}$ with $\mR_{ih}=\mathbbm 1\{\operatorname{type}(i)=h\}$, so that $\mathbf G=\mR\,\mathbf C\,\mE^\top$. The long-range states of all recent queries are the rows of
  \begin{equation}
    \mathbf G\,\mZ
    \;=\;\mR\,\mathbf C\,\mE^\top \mZ
    \;=\;\underbrace{\mR\,\big(\,\underbrace{\mathbf C\,\big(\,
          \underbrace{\mE^\top \mZ}_{\textstyle \mF\in\R^{q^{d-1}\times rp}}\,
          \big)}_{\textstyle U\in\R^{B\times rp}}\,\big)}_{\textstyle \in\R^{m\times rp}},
    \label{eq:assoc}
  \end{equation}
and Algorithm~\ref{alg:prefill} evaluates \eqref{eq:assoc} from the right: $\mE^\top Z$ is the pooling step, $\mathbf C(\cdot)$ the tabulation step, and the outer $\mR$ the per-type read, which is a row gather whose inverse permutation is the scatter back into chronological order.

For the causal branch, fix $i\in I_b$. Because $n$ is chunk-aligned, $I_b$ lies
entirely in $[n]$ or entirely in $(n,T]$. The incoming state collects the keys
of all earlier chunks in that block and the within-chunk term collects
$\{j\in I_b:j\le i\}$, so their union is $\{j\le i\}$ when $i\le n$ and
$\{n<j\le i\}$ when $i>n$. These are the row supports of $\mathbf L_n$
and $\mathbf L_m$. Summing the branches and dividing once gives the output for $\mathbf M^{(d)}$.

For cost, the three factors of \eqref{eq:assoc} correspond to three counts, each obtained from the last by collapsing one index, each worth a factor of $\Theta(T^{1/d})$. Note, $m,n=\Theta(T)$, $B=\Theta(q^{d-1})$ and $q=\Theta(T^{1/d})$.

From $\mathbf G$ itself, we have a hyperplane of $\mathbb F_q^{d-1}$ containing $q^{d-2}$ of the $q^{d-1}$ points (a $1/q$ of the entries are $1$s), so a recent query is incident to $\Theta(n/q)$ distant keys and
  \begin{equation}
    \operatorname{nnz}(\mathbf G)=\Theta\big(mn/q\big)=\Theta\big(T^{2-1/d}\big),
    \label{eq:nnzG}
  \end{equation}
  Applying $\mR$ collapses the queries -- row $i$ of $\mathbf G$ depends on $i$ only through
  $\operatorname{type}(i)$ -- so the $m$ rows take only $B$ distinct values, and the $(\text{type},\text{token})$ incidences number 
  \[\Theta\big(Bn/q\big)=\Theta\big(nq^{d-2}\big)=\Theta(T^{2-2/d}),\] 
  which produces an additional saving of $m/B=\Theta(T^{1/d})$. Applying $E^\top$ collapses the keys in the same way: column $j$ depends on $j$ only through $\operatorname{prof}(j)$, so the $n$ columns take only $q^{d-1}$ distinct values, and thus the $(\text{type},\text{profile})$ incidences are the nonzeros of $\mathbf C$,
  \begin{equation}
    \operatorname{nnz}(\mathbf C)=\Theta\big(Bq^{d-2}\big)
    =\Theta\big(q^{2d-3}\big)=\Theta\big(T^{2-3/d}\big),
    \label{eq:nnzC}
  \end{equation}
  a further saving of $n/q^{d-1}=\Theta(T^{1/d})$. Thus, tabulation performs one state addition of size $rp$ per nonzero of $\mathbf C$, for $O(T^{2-3/d}rp)$.

  The remaining steps are linear in $T$. Pooling reads each distant token once and adds it into one bucket; the per-type contractions cost $rp$ per recent query, for $\sum_h m_h\,rp=m\,rp$ in total, independent of $B$; the chunk summaries and the scan touch each token and each of the $T/c$ chunk boundaries a constant number of times. Together these are $O(Trp)$. Within a chunk, calculating the local attention scores costs $O(c^2r)$ and applying them to the augmented values costs $O(c^2p)$; summing over $T/c$ chunks gives $O(Tc(r+p))$. Together, we have the stated bound.

\textbf{Backward Pass.} For the Long-range branch, let the three steps of \eqref{eq:schedule} be
\[
  \begin{aligned}
    \mathcal P:\mZ\mapsto F,&\qquad
      F_x=\!\!\!\sum_{j\le n:\,\operatorname{prof}(j)=x}\!\!\!\mZ_j,\\[2pt]
    \mathcal C:F\mapsto U,&\qquad U_h=\sum_x\mathbf C_{hx}F_x,\\[2pt]
    \mathcal R_\Phi:U\mapsto\bar y^{\mathrm{lr}},&\qquad
      \bar y^{\mathrm{lr}}_{n+i}=\phi_{n+i}^\top U_{\operatorname{type}(i)},
  \end{aligned}
\]
so that $\bar y^{\mathrm{lr}}=\mathcal R_\Phi\,\mathcal C\,\mathcal P\, \mathbf Z$. 

Let
\[
\bar{g}_i := \frac{\partial \mathcal{L}}{\partial \bar{y}_i}
\in \mathbb{R}^p
\]
denote the gradient of the loss with respect to the augmented output. Pairing with $\bar g$ and moving one factor at a time across the inner product
gives
\begin{equation}
  \bar U_h=\!\!\!\sum_{i:\,\operatorname{type}(i)=h}\!\!\!\phi_{n+i}\,\bar g_{n+i}^\top,
  \quad
  \bar F=\mathbf C^\top\bar U,
  \quad
  \bar \mZ_j=\bar F_{\operatorname{prof}(j)},
  \quad
  \bar\phi_{n+i}=U_{\operatorname{type}(i)}\,\bar g_{n+i}.
  \label{eq:lr-adjoint}
\end{equation}
The two grouped steps switch roles: the forward pass reduces over $\operatorname{prof}$ and broadcasts over $\operatorname{type}$, while \eqref{eq:lr-adjoint} reduces over
$\operatorname{type}$ and broadcasts over $\operatorname{prof}$. Each is one
pass over the tokens it owns at $O(rp)$ per token, so both cost $O(Trp)$.

For the middle step, the incidence structure is biregular, i.e. every hyperplane of $\mathbb F_q^{d-1}$ contains $q^{d-2}$ points, and also every point lies on $(q^{d-1}-1)/(q-1)=B/q$ hyperplanes (one per normalized direction $a$), since $x\in H_{a,b}$ forces $b=a^\top x$. Hence, $\operatorname{nnz}(\mathbf C^\top)=\operatorname{nnz}(\mathbf C)=\Theta(T^{2-3/d})$,
and $\mathbf C^\top$ has constant column degree just as $\mathbf C$ has constant row degree. The tabulation step therefore applies to $\mathbf C^\top$ with the point-to-hyperplane table in place of the hyperplane-to-point one, giving $O(T^{2-3/d}rp)$ for $\bar F$, matching the forward calculation.

For the causal branch, fix a chunk $I_b$ and stack its rows as $\Phi_b\in\R^{c\times r}$, $\Psi_b\in\R^{c\times r}$, $\bar V_b\in\R^{c\times p}$, and write $A^{(b)}:=\operatorname{tril}(\Phi_b\Psi_b^\top)$ for the within-chunk score tile, so that \eqref{eq:local} reads
$Y_b=\Phi_b S^{\mathrm{in}}_b+A^{(b)}\bar V_b$. Differentiating this at fixed $\operatorname{tril}$ pattern,
\begin{equation}
  \begin{aligned}
    \bar S^{\mathrm{in}}_b&=\Phi_b^\top\bar Y_b, &\qquad
    \bar{\bar V}_b&=\big(A^{(b)}\big)^\top\bar Y_b,\\[2pt]
    \bar\Phi_b&=\bar Y_b\big(S^{\mathrm{in}}_b\big)^\top
               +\operatorname{tril}\!\big(\bar Y_b\bar V_b^\top\big)\Psi_b,
    &\qquad
    \bar\Psi_b&=\operatorname{tril}\!\big(\bar Y_b\bar V_b^\top\big)^{\!\top}\Phi_b.
  \end{aligned}
  \label{eq:local-adjoint}
\end{equation}
Each is one $c\times c$ by $c\times r$ or $c\times p$ contraction, so
$O(c^2(r+p))$ per chunk and $O(Tc(r+p))$ over the $T/c$ chunks, the same as the
forward tile. Note that $A^{(b)}$ and $\operatorname{tril}(\bar Y_b\bar V_b^\top)$
are rebuilt from $\Phi_b,\Psi_b,\bar Y_b,\bar V_b$ when the chunk is visited, so
one $c^2$ tile is live at a time (same as in the forward pass).

For the scan, $S^{\mathrm{in}}_b=\sum_{a<b}D_a$ gives
\[
  \sum_b\big\langle\bar S^{\mathrm{in}}_b,\textstyle\sum_{a<b}D_a\big\rangle
  =\sum_a\Big\langle\textstyle\sum_{b>a}\bar S^{\mathrm{in}}_b,\;D_a\Big\rangle\]
  so $\bar D_a=\sum_{b>a}\bar S^{\mathrm{in}}_b$. In other words, 
the adjoint of an exclusive prefix scan is an exclusive suffix scan over the
same chunks, run separately within $[n]$ and within $(n,T]$, at the same $O(T/c)$
state additions. Broadcasting $\bar D_b$ back to the tokens of $I_b$ and adding
the two contributions of \eqref{eq:lr-adjoint} and \eqref{eq:local-adjoint}
accumulates $\bar \mZ_j$, from which $\bar\psi_j=\bar \mZ_j\bar v_j$ and
$\bar{\bar v}_j=\bar \mZ_j^\top\psi_j$ follow pointwise in $O(rp)$ per token.

Finally since $o_i=\bar y_{i,1:d_v}/\bar y_{i,p}$ with $\bar y_{i,p}>0$ by assumption, its Jacobian is row-local: with $\bar o_i:=\frac{\partial\mathcal L}{\partial o_i}$,
\[
  \bar g_{i,1:d_v}=\frac{\bar o_i}{\bar y_{i,p}},
  \qquad
  \bar g_{i,p}=-\frac{\bar o_i^\top\,\bar y_{i,1:d_v}}{\bar y_{i,p}^{2}}
             =-\frac{\bar o_i^\top o_i}{\bar y_{i,p}},
\]
at $O(p)$ per token and $O(Tp)$ overall, which is dominated.

Every step of Algorithm~\ref{alg:prefill} is therefore matched by an adjoint of the same shape and arithmetic. The backward pass runs in $O\big((T^{2-3/d}+T)rp+Tc(r+p)\big)$ work and
$O\big((T^{1-1/d}+T/c)rp+c^2\big)$ words, the bounds of \eqref{eq:cost}.\qedhere\\
\end{proof}

\begin{lemma}[Gated SMat]
\label{thm:gated}
Under \eqref{eq:gated} with scalar per-token gates, Algorithm~\ref{alg:prefill} computes the exact outputs within the complexity bound of \eqref{eq:cost}, using $O(T)$ additional space.
\end{lemma}

Lemmas~\ref{thm:gated}--~\ref{lem:decode} concern additive kernel attention with optional
scalar decay; they do not cover the delta-rule recurrence
used in the GDN extension of Section~\ref{sec:recurrent-extensions}.

\begin{proof}
We first compute $\lambda_1,\dots,\lambda_T$ with one prefix sum, in $O(T)$ work and $O(T)$ space. Every factor used afterwards has the form $\exp(\lambda_i-\lambda_j)$ with $j\le i$. Since $\lambda$ is non-increasing, each factor lies in $(0,1]$, so we never form $\exp(\lambda_i)$ or $\exp(-\lambda_j)$ separately and no step can overflow.

\emph{Causal branch.} Let $e(b)$ be the last index of chunk $I_b$. Each entry of the within-chunk tile becomes $(\phi_i^\top\psi_j)\exp(\lambda_i-\lambda_j)$, which is one entrywise product on a tile that is already materialized. The chunk summary becomes $D_b=\sum_{j\in I_b}\exp(\lambda_{e(b)}-\lambda_j)\mZ_j$, which rescales each $\mZ_j$ by a scalar prior to the sum. The carried state is read as $\exp(\lambda_i-\lambda_{e(b-1)})\,\phi_i^\top S^{\mathrm{in}}_b$ and updated as $S^{\mathrm{in}}_{b+1}=\exp(\lambda_{e(b)}-\lambda_{e(b-1)})\,S^{\mathrm{in}}_b+D_b$, adding one scalar multiplication per chunk. The scan is otherwise unchanged.

\emph{Long-range branch.} For a distant key $j\le n$ and a recent query $n+i$, the factor splits at the boundary,
\[
\exp(\lambda_{n+i}-\lambda_j)
=\exp(\lambda_{n+i}-\lambda_n)\cdot\exp(\lambda_n-\lambda_j),
\]
into a query-side and a key-side scalar, each at $(0,1]$. The key-side scalar is folded into $\mZ_j$ before pooling, and the query-side scalar multiplies the query's read from the type table. Neither modifies $\mathbf C$, the profile table or the type table, so the long-range branch is the ungated computation applied to rescaled inputs.

The added work is $O(T)$ exponentials, $O(Tc)$ for the tile products and $O(Tr)$ to rescale features, all dominated by terms already in \eqref{eq:cost}. The added space is the $O(T)$ values of $\lambda$.
\end{proof}

Note, the same factorization gives streaming decoding. Let $a_t=\exp(\lambda_t-\lambda_{t-1})\in(0,1]$. At the boundary, cache the type-major states
\[
U_h^{(n)}=\sum_{j\le n} C_{h,\operatorname{prof}(j)}\,
\exp(\lambda_n-\lambda_j)\,\mZ_j ,
\]
which is the ungated cache with each $\mZ_j$ rescaled. For $t>n$, maintain the recent state $S_t=a_tS_{t-1}+\mZ_t$ and the scalar $g_t=a_tg_{t-1}$, starting from $S_n=0$ and $g_n=1$. By induction, $S_t=\sum_{n<j\le t}\exp(\lambda_t-\lambda_j)\mZ_j$ and  $g_t=\exp(\lambda_t-\lambda_n)$, so the augmented output is
\[
\bar y_t=\phi_t^\top\!\left[\,g_t\,U^{(n)}_{\operatorname{type}(t-n)}+S_t\right].
\]
Each token costs the same as as the ungated computation plus two scalar multiplications.

\section{Decoding and Horizon-Free Decoding}
\label{sec: decoding and horizon free decoding}
We first restate and prove Theorem \ref{thm:decode}.
\streamingdecoding*
\begin{proof}
The type table $\{U_h\}_{h<B}$ depends only on the distant
tokens. Since every distant position precedes every recent one, it is fixed
once the distant block has been consumed, and hence, the profile states $\{F_x\}$ can be discarded. By the decomposition in the proof of Theorem~\ref{thm:exact}, the augmented output at recent position $n+i$ is
\begin{equation}\bar y_{n+i}
  =\phi_{n+i}^\top\Big(U_{\operatorname{type}(i)}+\!\!\sum_{n<j\le n+i}\!\!\mZ_j\Big),
\end{equation}
and the second term is a single running state $S$ maintained by the in-place
update $S\leftarrow S+\mZ_{n+i}$. The index
$\operatorname{type}(i)=i\bmod B$ is one arithmetic operation. Forming
$\mZ_{n+i}$, updating $S$, and contracting
$\phi_{n+i}^\top(U_{\operatorname{type}(i)}+S)$ each cost $O(rp)$. The retained state is the $B$ tabulated states together with $S$. For the gated variant from Equation \ref{eq:gated}, the running state update becomes $S\leftarrow a_{n+i}S+\mZ_{n+i}$ and the read of $U_{\operatorname{type}(i)}$ carries the scalar $\exp(\lambda_{n+i}-\lambda_n)$, maintained by one addition per token.\qedhere
\end{proof}

\label{app:window}

The streaming-decoding theorem (Theorem~\ref{thm:decode}) assumes that the total length $T$ of the
token sequence is known in advance, since the distant/recent boundary $n$ and the field size $q$ are both functions of $T$ and together determine the structure of the mask. Hence, if more tokens were to be added beyond $T$, the layer degrades to linear attention. In this section we replace that assumption with a fixed training horizon $T_{\max}$ and a recent window that advances in steps of $c_{dc}$.

\begin{definition}[Stepped-window mask]
\label{def:window}
Fix a training horizon $T_{\max}$, a block length $c_{dc}$ dividing $T_{\max}$ and a multiple of the chunk
width $c$, and the geometry $(q,\mC,B)$ from the SMat construction of
Section~\ref{subsec: mask constructions} for $T=T_{\max}$. For $t\ge1$ let
$b(t)=\lfloor (t-1)/c_{dc}\rfloor$ be the block of position $t$ and
$n(t)=\max\{0,\,(b(t)-1)c_{dc}\}$ be its number of distant positions. For any $T$, including
$T>T_{\max}$,
\[
  \mM^{(\mathrm{step})}_{tj}=
  \begin{cases}
    1 & n(t)<j\le t \quad(\text{recent window}),\\
    \mC_{\operatorname{type}(t),\,\operatorname{prof}(j)} & j\le n(t) \quad(\text{distant, through }\mG),\\
    0 & j>t.
  \end{cases}
\]
\end{definition}

In the block form of Section~\ref{subsec: mask constructions} the boundary $n=c\lfloor T/2c\rfloor$ is a constant determined by $T$. Here, the boundary is the function $n(\cdot)$, of which the constant $n$ is the special case $n(t)\equiv n$. Since $n(t)$ depends only on $t$, the mask for a length-$t$ sequence is the leading $t\times t$ block of the mask for any longer sequence. Thus, training at $T_{\max}$ and decoding at any length uses one mask.

In Algorithm~\ref{alg:decode}, each fresh recurrent state is initialized
to zero, with update $R \leftarrow R + Z_t$ and readout
$\operatorname{read}(R,\phi_t)=\phi_t^\top R$.
For the scalar-gated variant, the update becomes
$R \leftarrow a_t R + Z_t$.

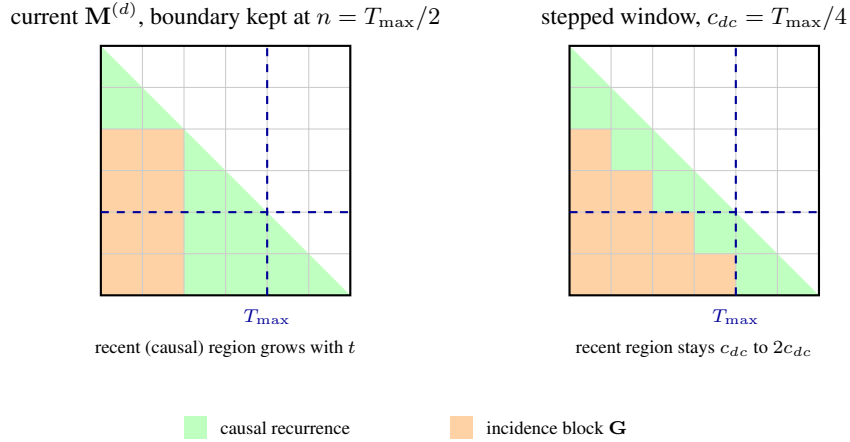
\begin{figure}[h]
\centering
\begin{tikzpicture}[x=0.55cm,y=0.55cm]
  \def\K{6}
  \foreach \b in {0,...,5} {
    \pgfmathsetmacro{\yt}{\K-\b}\pgfmathsetmacro{\yb}{\K-\b-1}
    \ifnum\b>1
      \fill[orange!35] (0,\yb) rectangle (2,\yt);
      \ifnum\b>2 \fill[green!25] (2,\yb) rectangle (\b,\yt); \fi
    \else
      \ifnum\b=1 \fill[green!25] (0,\yb) rectangle (1,\yt); \fi
    \fi
    \fill[green!25] (\b,\yt) -- (\b,\yb) -- (\b+1,\yb) -- cycle;
  }
  \draw[step=1,gray!40] (0,0) grid (6,6);
  \draw[thick] (0,0) rectangle (6,6);
  \draw[dashed,thick,blue!60!black] (4,0) -- (4,6);
  \draw[dashed,thick,blue!60!black] (0,2) -- (6,2);
  \node[above] at (3,6.15) {\small current $\mM^{(d)}$, boundary kept at $n=T_{\max}/2$};
  \node[blue!60!black,below] at (4,-0.1) {\scriptsize $T_{\max}$};
  \node[below] at (3,-0.85) {\scriptsize recent (causal) region grows with $t$};
  \begin{scope}[xshift=6.2cm]
  \foreach \b in {0,...,5} {
    \pgfmathsetmacro{\yt}{\K-\b}\pgfmathsetmacro{\yb}{\K-\b-1}
    \ifnum\b>1 \fill[orange!35] (0,\yb) rectangle (\b-1,\yt); \fi
    \ifnum\b>0 \fill[green!25] (\b-1,\yb) rectangle (\b,\yt); \fi
    \fill[green!25] (\b,\yt) -- (\b,\yb) -- (\b+1,\yb) -- cycle;
  }
  \draw[step=1,gray!40] (0,0) grid (6,6);
  \draw[thick] (0,0) rectangle (6,6);
  \draw[dashed,thick,blue!60!black] (4,0) -- (4,6);
  \draw[dashed,thick,blue!60!black] (0,2) -- (6,2);
  \node[above] at (3,6.15) {\small stepped window, $c_{dc}=T_{\max}/4$};
  \node[blue!60!black,below] at (4,-0.1) {\scriptsize $T_{\max}$};
  \node[below] at (3,-0.85) {\scriptsize recent region stays $c_{dc}$ to $2c_{dc}$};
  \end{scope}
  \begin{scope}[yshift=-1.9cm,xshift=1.1cm]
    \fill[green!25] (0,0) rectangle (0.55,0.55);
    \node[right] at (0.65,0.28) {\scriptsize causal recurrence};
    \fill[orange!35] (6.4,0) rectangle (6.95,0.55);
    \node[right] at (7.05,0.28) {\scriptsize incidence block $\mG$};
  \end{scope}
\end{tikzpicture}
\caption{Both masks drawn at the same length $T=1.5\,T_{\max}$, in blocks of $c_{dc}=T_{\max}/4$ tokens. \textbf{Left:} the fixed-$T$ mask once its boundary is kept at $n=T_{\max}/2$. The incidence block $\mG$ stops growing at two blocks and everything after the boundary falls to the recurrence. \textbf{Right:} the stepped window,
whose recent region stays between $c_{dc}$ and $2c_{dc}$ tokens at every length while $\mG$ keeps
absorbing the older blocks.}
\label{fig:masks_window}
\end{figure}

\begin{algorithm}[hbt!]
\caption{Stepped-window decoding for one attention head}
\label{alg:decode}
\begin{algorithmic}[1]
\REQUIRE geometry $(q,\mC,B)$ fixed from $T_{\max}$; block length $c_{dc}$; maps
$\operatorname{prof},\operatorname{type}$
\STATE $U\gets 0$;\quad $\Delta_{\mathrm{old}},\Delta_{\mathrm{cur}}\gets 0$;\quad
$R_{\mathrm{old}},R_{\mathrm{cur}}\gets$ fresh recurrent states
\FOR{$t=1,2,\dots$}
  \IF{$t>1$ \textbf{and} $b(t)\ne b(t-1)$}
    \STATE \label{ln:fold} \textbf{if} $b(t)\ge 2$ \textbf{then} $U\gets U+\mC\,\Delta_{\mathrm{old}}$
    \STATE \label{ln:shift} $\Delta_{\mathrm{old}} \gets\Delta_{\mathrm{cur}}$;\quad $\Delta_{\mathrm{cur}}\gets 0$;\quad
$R_{\mathrm{old}}\gets R_{\mathrm{cur}}$;\quad $R_{\mathrm{cur}}\gets$ fresh
  \ENDIF
  \STATE \label{ln:scatter} Compute $\phi_t$ and $\mZ_t=\psi_t\bar v_t^\top$; update
$R_{\mathrm{old}},R_{\mathrm{cur}}$; $\Delta_{\mathrm{cur}}[\operatorname{prof}(t)]\mathrel{+}=\mZ_t$.
  \STATE \label{ln:read} $\bar y_t\gets \phi_t^\top U_{\operatorname{type}(t)}+\operatorname{read}(R_{\mathrm{old}},\phi_t)$;\quad
$o_t\gets \bar y_{t,1:d_v}/\bar y_{t,p}$.
\ENDFOR
\end{algorithmic}
\end{algorithm}

\begin{lemma}[Horizon-free streaming decoding]
\label{lem:decode}
Assume the hypotheses of Theorem~\ref{thm:exact}, with the geometry fixed from $T_{\max}$ and
$c_{dc}=T_{\max}/4$. Then Algorithm~\ref{alg:decode} computes the outputs of the stepped-window
mask of Definition~\ref{def:window}, for every length $T$ at
$O\big((1+T_{\max}^{1-3/d})rp\big)$ work per token from a cache of $O(T_{\max}^{1-1/d})$ states of size $r\times p$.
\end{lemma}

\begin{proof}
Index blocks so that block $\beta$ occupies positions $\beta c_{dc}+1,\dots,(\beta+1)c_{dc}$, and write
$\beta(t)=b(t)$. For $0\le n\le T$ let
\[
  F^{(n)}_x:=\sum_{j\le n,\ \operatorname{prof}(j)=x}\mZ_j ,\qquad x\in\mathbb F_q^{d-1},
\]
be the profile table of the first $n$ positions, viewed as a $q^{d-1}\times rp$ matrix. Since
$n(t)=(\beta(t)-1)c_{dc}$ for $\beta(t)\ge1$, the distant region of a query in block $\beta$ is the union of blocks $0,\dots,\beta-2$.

We claim that for every $t\ge1$, by line~\ref{ln:read} in the algorithm,
\begin{enumerate}[label=(\roman*),leftmargin=2.2em,itemsep=0pt]
  \item $U=\mC F^{(n(t))}$;
  \item $\Delta_{\mathrm{old}}$ holds the per-profile sums of block $\beta(t)-1$, and $\Delta_{\mathrm{cur}}$
        those of block $\beta(t)$ restricted to positions $\le t$;
  \item $R_{\mathrm{old}}$ is the state of the recurrence initialized at position $n(t)+1$ and advanced
        through $t$.
\end{enumerate}
We prove this by induction on $t$. For $t=1$ we have $\beta(1)=0$ and $n(1)=0$, the conditional does not
fire, and $U=0=\mC F^{(0)}$, giving (i)--(iii). Assume the invariant at $t-1$ and consider $t$.

If $\beta(t)=\beta(t-1)$ then $n(t)=n(t-1)$ and lines~4--5 are skipped, so $U$ is unchanged and (i)
persists; line~\ref{ln:scatter} adds $\mZ_t$ to $\Delta_{\mathrm{cur}}[\operatorname{prof}(t)]$ and advances both
recurrences, which preserves (ii) and (iii).

If $\beta(t)\ne\beta(t-1)$ then $t=\beta c_{dc}+1$ and $t-1$ is the last position
of block $\beta-1$. For $\beta\le1$ we have $n(t)=0$, the guard on line~\ref{ln:fold} suppresses the update, and
(i) holds with $U=0$. For $\beta\ge2$, the inductive hypothesis at $t-1$ gives
$U=\mC F^{(n(t-1))}=\mC F^{((\beta-2)c_{dc})}$ and, by (ii), $\Delta_{\mathrm{old}}$ equal to the
per-profile sums of block $\beta-2$. Those sums are precisely
$F^{((\beta-1)c_{dc})}-F^{((\beta-2)c_{dc})}$, so line~\ref{ln:fold} yields
\[
  U \;=\; \mC F^{((\beta-2)c_{dc})}+\mC\Big(F^{((\beta-1)c_{dc})}-F^{((\beta-2)c_{dc})}\Big)
    \;=\; \mC F^{((\beta-1)c_{dc})} \;=\; \mC F^{(n(t))},
\]
which shows (i). Line~\ref{ln:shift} then sets $\Delta_{\mathrm{old}}$ to the sums of block $\beta-1$ and clears
$\Delta_{\mathrm{cur}}$, and line~\ref{ln:scatter} adds $\mZ_t$, giving (ii); and it reassigns
$R_{\mathrm{old}}\gets R_{\mathrm{cur}}$, where $R_{\mathrm{cur}}$ was initialized at the first position
of block $\beta-1$, namely $(\beta-1)c_{dc}+1=n(t)+1$, giving (iii).

\emph{Exactness.} The augmented output is linear in the mask row, so by
Definition~\ref{def:window},
\begin{equation}
  \bar y_t=\sum_{j}\mM^{(\mathrm{step})}_{tj}\,\phi_t^\top \mZ_j
  =\phi_t^\top\Big(\underbrace{\sum_{j\le n(t)}\mC_{\operatorname{type}(t),\operatorname{prof}(j)}\mZ_j}_{\text{distant}}
  +\underbrace{\sum_{n(t)<j\le t}\mZ_j}_{\text{recent}}\Big).
  \label{eq:window-split}
\end{equation}
Grouping the distant keys by profile gives
$\sum_{j\le n}\mC_{h,\operatorname{prof}(j)}\mZ_j=(\mC F^{(n)})_h$, so by invariant~(i) the first
term of \eqref{eq:window-split} equals $\phi_t^\top U_{\operatorname{type}(t)}$. The recurrence is
linear and its state is initialized to zero at $n(t)+1$, so by invariant~(iii) it holds
$\sum_{n(t)<j\le t}\mZ_j$ and the second term equals $\operatorname{read}(R_{\mathrm{old}},\phi_t)$.
Both are accumulated into $\bar y_t$ before the single division on line~\ref{ln:read}, so the normalized output
is exact.

\emph{Gating.} For the gated mask $\widetilde{\mM}_{tj}=\mM^{(\mathrm{step})}_{tj}\exp(\lambda_t-\lambda_j)$,
replace $F^{(n)}$ by the table referred to the last distant position,
$\widetilde F^{(n)}_x:=\sum_{j\le n,\ \operatorname{prof}(j)=x}\exp(\lambda_n-\lambda_j)\mZ_j$, and read
$U$ as $\exp(\lambda_t-\lambda_{n(t)})\,\phi_t^\top U_{\operatorname{type}(t)}$. Writing
$N=(\beta-2)c_{dc}$ and $N'=(\beta-1)c_{dc}$, the two tables satisfy
\[
  \widetilde F^{(N')}=\exp(\lambda_{N'}-\lambda_N)\,\widetilde F^{(N)}
  +\Big(\text{per-profile sums of block }\beta-2\text{ referred to }N'\Big),
\]
so invariant~(i) is restored with line~\ref{ln:fold} replaced by
$U\gets\exp(\lambda_{N'}-\lambda_N)U+\mC\,\Delta_{\mathrm{old}}$, one scalar multiply per block. The
induction and the exactness argument are otherwise unchanged.

\emph{Cost.} Per token, forming $\phi_t$ and $\mZ_t$, the two recurrent updates, the scatter-add into
$\Delta_{\mathrm{cur}}$, and the read and normalization on line~\ref{ln:read} are each $O(rp)$. The only
other work is line~\ref{ln:fold}, one sparse product costing $\operatorname{nnz}(\mC)\,rp$ scalar operations once
every $c_{dc}$ tokens. As written, line~\ref{ln:fold} performs the sparse product at the first token of each block, so the algorithm meets the bound below in amortized form. For a worst-case bound, the next type table $U+\mC\,\Delta_{\mathrm{old}}$ is built in a second buffer during the preceding block: the block it absorbs is complete for all $c_{dc}$ steps of that block, so copying $U$ ($B\,rp$ operations) and the sparse product ($\operatorname{nnz}(\mC)\,rp$) can be spread over those steps, and the buffers are swapped only once the new table is complete. Since $B\le\operatorname{nnz}(\mC)$, this adds $O\big((1+\operatorname{nnz}(\mC)/c_{dc})\,rp\big)$ work to each token. In the regular construction, the number of normalized directions is $|\mathcal A|=(q^{d-1}-1)/(q-1)=\Theta(q^{d-2})$ and each profile lies on one hyperplane per direction, so \[\operatorname{nnz}(\mC)=|\mathcal A|\,q^{d-1}=\Theta(q^{2d-3}).\]
With $T_{\max}=\Theta(q^{d})$ this is $\operatorname{nnz}(\mC)=\Theta(T_{\max}^{2-3/d})$, and since $c_{dc}=\Theta(T_{\max})$ the per-token share is $\Theta(T_{\max}^{1-3/d})$, giving
$O\big((1+T_{\max}^{1-3/d})rp\big)$, which is $O(rp)$ exactly when $d\le3$.

\emph{Cache.} The retained state is the type table $U$ ($B$ states), the two profile deltas
($q^{d-1}$ states each), the two recurrent states, and the second type table used above. Since
$B=\Theta(q^{d-1})=\Theta(T_{\max}^{1-1/d})$, the total is $ O(T_{\max}^{1-1/d})$ states of
size $r\times p$.
\end{proof}

\paragraph{Other dynamic boundary candidates.}
Write $n(t)$ for the boundary each schedule places at query $t$, and $T_{\mathrm{train}}$ for the
training length. We compare six boundary schedules. 
\begin{enumerate}
    \item \emph{No long-range branch} sets $n(t)=0$, so $\mG$ is never used and the layer is the ordinary recurrence, the $d=1$ SMat member.
    \item \emph{Boundary kept, current} is the mask of the main text with $n(t)=T_{\mathrm{train}}/2$ held where prefill placed it, which is what the fixed-horizon decoder does when it is run past its horizon.
    \item \emph{Rebuilt, current} is the same mask recomputed at the evaluation length, $n(t)=T/2$. It is computationally expensive and needs $T$ in advance.
    \item \emph{Doubling} places the boundary at $n(t)=2^{\lfloor\log_2 t\rfloor-1}$, so it doubles at each power of two. 
    \item The two \emph{window} arms are definition~\ref{def:window} with $c_{dc}=T_{\mathrm{train}}/4$ and $c_{dc}=T_{\mathrm{train}}/8$.
\end{enumerate}
All arms share architecture, data, optimizer and step budget. We demonstrate capabilities with the $d=2$ SMat mask: a
Mamba-2-style decayed linear recurrence reset at $n(t)$, plus the incidence branch read by a fixed random hash of token identifiers of Section~\ref{subsec: mask constructions}. Each result is the mean over three seeds, with the sample standard deviation as a subscript.

\paragraph{Task and evaluation layouts.}
Each MQAR sequence contains $64$ key--value pairs over a noise vocabulary. Every key occurs twice,
once beside its value and once as a query, and the model must emit that value at the query. During
training the pair and query positions are drawn uniformly over the sequence. At evaluation we use two
layouts. Under \emph{uniform} the positions are drawn as in training, so the result measures length
generalization alone. Under \emph{placed} all pairs are confined to a $256$-token region and all
queries to a $256$-token region at the very end of the sequence, separated by a gap of $D$ tokens;
$D$ controls how far back a query must reach, with large $D$ putting the pairs near the start of the
context. The placed layout moves the stored content across each schedule's boundary while holding everything else fixed.

 \paragraph{Length generalization.}
  Table~\ref{tab:mqar-uniform} evaluates each schedule out to eight times
  its training length under the
  training distribution.

  \begin{table}[t]
  \centering
  \small
  \begin{tabular}{lrrrr}
  \hline
  Schedule & 1024 & 2048 & 4096 & 8192 \\
  \hline
  No long-range branch & 13.3 (10.0) & 13.5 (10.6) & 12.8 (9.6) & 11.7 (9.4)
  \\
  Boundary kept, current & 91.5 (3.7) & 78.4 (5.3) & 63.6 (9.2) & 51.9
  (11.2) \\
  Rebuilt, current & 91.6 (4.0) & 87.2 (4.7) & 83.7 (6.1) & 81.8 (6.8) \\
  Doubling & 78.8 (4.0) & 68.7 (4.2) & 62.6 (1.8) & 53.9 (1.5) \\
  \rowcolor{gray!12}
  Window, $c_{dc}=T_{\max}/8$ & 93.8 (4.9) & 98.3 (1.4) & 95.6 (6.8) & 90.1
  (13.9) \\
  \rowcolor{gray!12}
  Window, $c_{dc}=T_{\max}/4$ & 82.2 (2.3) & 92.9 (1.0) & 95.9 (0.5) & 97.6
  (1.1) \\
  \hline
  \end{tabular}
  \caption{MQAR accuracy (\%) by evaluation length, \emph{uniform} layout,
  models trained at length 1024. The quarter-length window is the only schedule whose accuracy rises with the evaluation length, and the only one whose spread does not. Knowing $T$ in advance recovers part but not all of the fixed boundary's loss: the rebuilt oracle ends at 81.8 against 51.9 for the kept boundary. Mean (std) over 3 seeds.}
  \label{tab:mqar-uniform}
  \end{table}

  \paragraph{Recall relative to the boundary.}
  Table~\ref{tab:mqar-placed} fixes the evaluation length at $8192$ and
  sweeps the gap $D$. The final
  column, $D=7678$, places the pairs at the very start of the sequence,
  where every schedule routes
  them through $\mG$; it is a control rather than a hard case.

  \begin{table}[t]
  \centering
  \small
  \begin{tabular}{lrrrrr}
  \hline
  Schedule & $D=0$ & $D=256$ & $D=1024$ & $D=3072$ & $D=7678$ \\
  \hline
  No long-range branch & 8.7 (6.1) & 7.9 (5.7) & 8.3 (6.3) & 8.0 (6.4) & 7.0
  (5.8) \\
  Boundary kept, current & 35.4 (14.9) & 34.7 (15.7) & 34.1 (15.1) & 33.6
  (14.7) & 99.6 (0.3) \\
  Rebuilt, current & 36.1 (18.4) & 34.9 (18.0) & 35.9 (18.9) & 34.6 (18.8) &
  99.7 (0.3) \\
  Doubling & 22.4 (3.8) & 23.1 (2.4) & 23.0 (3.2) & 22.1 (3.1) & 85.6 (7.5)
  \\
  \rowcolor{gray!12}
  Window, $c_{dc}=T_{\max}/8$ & 73.4 (20.2) & 81.9 (23.0) & 81.7 (22.5) &
  81.9 (22.9) & 83.9 (20.1) \\
  \rowcolor{gray!12}
  Window, $c_{dc}=T_{\max}/4$ & 27.1 (3.0) & 94.9 (3.2) & 94.3 (3.3) & 94.7
  (3.0) & 97.5 (2.6) \\
  \hline
  \multicolumn{6}{l}{\emph{Fraction of query rows whose pair region is
  routed through $\mG$}} \\
  Fixed boundary / doubling & 0.00 & 0.00 & 0.00 & 0.00 & 1.00 \\
  \rowcolor{gray!12}
  Window, $c_{dc}=T_{\max}/4$ & 0.00 & 1.00 & 1.00 & 1.00 & 1.00 \\
  \hline
  \end{tabular}
  \caption{MQAR accuracy (\%) by gap $D$ between the pair region and the
  query region, \emph{placed}
  layout at evaluation length $8192$. The lower block reports how often the
  pairs actually reach the
  long-range branch. At $D=7678$ the pairs precede every boundary and the
  fixed schedules are the
  strongest arms. For $256\le D\le3072$ the pairs land inside the fixed
  boundary's recent region,
  which it can serve only from the fixed-size recurrence, and its accuracy
  falls to about 34\% while
  the quarter-length window holds about 95\%. At $D=0$, the quarter-length window obtains $27.1\%$, below the two fixed-boundary variants ($35.4\%$ and $36.1\%$). The eighth-length window instead obtains $73.4\%$, with substantial seed variation. Performance near the boundary therefore depends on the particular window schedule. Mean (std) over 3 seeds.}
  \label{tab:mqar-placed}
  \end{table}

  \paragraph{Language modelling.}
  Table~\ref{tab:pg19} reports byte-level PG-19, trained at length 2048 and
  evaluated on $16384$-byte
  windows.

  \begin{table}[t]
  \centering
  \small
  \begin{tabular}{lrrrr}
  \hline
  Schedule & $[0,1024)$ & $[2048,4096)$ & $[4096,8192)$ & $[8192,16384)$ \\
  \hline
  No long-range branch & 1.686 (0.006) & 1.675 (0.007) & 1.679 (0.008) &
  1.670 (0.007) \\
  Boundary kept, current & 1.634 (0.005) & 1.588 (0.005) & 1.598 (0.004) &
  1.595 (0.005) \\
  Rebuilt, current & 1.634 (0.005) & 1.614 (0.008) & 1.619 (0.006) & 1.601
  (0.004) \\
  Doubling & 1.646 (0.008) & 1.605 (0.008) & 1.610 (0.007) & 1.610 (0.008)
  \\
  \rowcolor{gray!12}
  Window, $c_{dc}=T_{\max}/8$ & 1.608 (0.003) & 1.546 (0.003) & 1.561
  (0.004) & 1.559 (0.004) \\
  \rowcolor{gray!12}
  Window, $c_{dc}=T_{\max}/4$ & 1.639 (0.008) & 1.584 (0.012) & 1.596
  (0.010) & 1.595 (0.014) \\
  \hline
  \end{tabular}
  \caption{PG-19 bits per byte by position within the evaluation window,
  trained at length 2048 and
  evaluated at $16384$; lower is better. Every schedule improves past the
  training length rather than
  degrading, and the quarter-length window is similar to the fixed boundary.
  Mean (std) over 3 seeds.}
  \label{tab:pg19}
  \end{table}

\section{Training Details and additional results}
\label{app:training}

All runs use an Adam-family optimizer, gradient-norm clipping at $1.0$, and a single GPU per run.
Learning-rate schedules are linear warmup followed by cosine decay to $0.1\times$ the peak, except on subset routing and multi-key subset recall, which train at a constant rate. Weight decay is applied to matrix parameters only; biases, norms, and the recurrent parameters $A_{\log}$, $\Delta_{\mathrm{bias}}$, and $D$ are excluded. Every arm within a table shares its data order, schedule, and precision.

\begin{table}[h]
\centering
\small
\setlength{\tabcolsep}{4pt}
\begin{tabular}{llllll}
\toprule
& \textbf{Routing} & \textbf{MQAR} & \textbf{MKAR} & \textbf{JCKR} & \textbf{PG-19} \\
\midrule
\multicolumn{6}{l}{\emph{Data}}\\
Sequence length      & 1024        & 64--256      & 1024        & 64--3076       & 16384 / 32768 \\
Vocabulary           & {--}        & 8192         & {--}        & 563            & 50257 (GPT-2) \\
Task size            & $k\le6$     & 4--64 pairs  & 8 pairs     & 4--512 records & {--} \\
\midrule
\multicolumn{6}{l}{\emph{Model}}\\
Layers               & 1           & 2            & 1           & 2              & 8 \\
Width $d_{\mathrm{model}}$ & 64    & 16/32/64     & 256         & 64             & 384 \\
Head dim.            & {--}        & 16           & 64          & 16             & 64 \\
State dim.           & {--}        & 16           & 64          & 16             & 64 \\
Short conv.\ width   & {--}        & 3            & 3           & 4              & 4 \\
\midrule
\multicolumn{6}{l}{\emph{Optimization}}\\
Optimizer            & Adam        & AdamW        & AdamW       & AdamW          & AdamW \\
Peak LR              & 3e-3        & 1e-2         & 3e-4        & 3e-3           & 6e-4 \\
Weight decay         & 0           & 0.1          & 0.1         & 0.1            & 0.1 \\
Batch size           & 64          & 256          & 32          & 256            & 2 / 1 \\
Budget               & 1500 steps  & 32 epochs    & 12000 steps & 32 epochs      & 9155 steps \\
Warmup               & 0           & {--}         & 0           & 0              & 300 \\
Precision            & fp32        & bf16         & bf16        & bf16           & bf16 \\
Seeds                & 3           & 3            & 3           & 3              & 1 \\
\bottomrule
\end{tabular}
\caption{Training configuration for each experiment. Routing trains with the mask
fixed and optimizes a mean-squared-error readout; MKAR uses binary cross-entropy on one sigmoid logit per pair (multi-label), and MQAR, JCKR and PG-19 use categorical (softmax) cross-entropy.
PG-19 batch sizes are per context length, chosen so that every step sees 32{,}768
tokens, giving 300M training tokens per arm.}
\label{tab:training}
\end{table}

\subsection{Extension Details}
\label{app:content-routing}

We describe one attention head and suppress the head index.
Let $h_t$ denote the layer input at position $t$, which depends
only on tokens at positions $1,\ldots,t$. The distant/recent
boundary $n$ and the finite geometry are fixed independently
of token content.

\textbf{Causal write addresses.}
The write-side representation is a short causal convolution,
\[
    c_t = \sum_{\ell=0}^{L-1} K_\ell h_{t-\ell},
\]
with zero padding before the sequence begins; our implementation
uses a depthwise convolution with $L=4$. This allows a write to
be addressed using nearby preceding content, such as a key
preceding its value. Write addresses are computed when tokens
are processed and are not revised using later queries. Let $D=d-1$. For each coordinate $k\in[D]$, we compute
\[
    z_{t,k}
    = \gamma_k
      \left\langle
        \frac{w_k}{\|w_k\|_2},
        \operatorname{LN}(c_t)
      \right\rangle + b_k,
    \] and
    \[s_{t,k}
    = \operatorname{clip}
      \bigl(q\Phi(z_{t,k}),\,0,\,q-\varepsilon\bigr),
\]
where $\operatorname{LN}$ normalizes each token independently,
$\Phi$ is the standard normal CDF, and $\varepsilon>0$ is a
small numerical constant. The discrete hash is
\[
    x_\theta(c_t)
    = \bigl(\lfloor s_{t,1}\rfloor,\ldots,
             \lfloor s_{t,D}\rfloor\bigr)
    \in \mathbb{F}_q^D.
\]
Thus, $\operatorname{prof}(j)=x_\theta(c_j)$ determines the
profile memory updated by distant token $j$. The projection
parameters $(w_k,\gamma_k,b_k)$ are learned jointly with the
backbone; the geometric incidence matrix $C$ remains fixed.

\textbf{Learning the discrete hash.}
The floor operation has no useful ordinary derivative. Therefore, following \cite{Jiang_2018}, we use a straight-through estimator: the forward
pass uses a discrete address, while the backward pass uses
a surrogate based on interpolation between neighboring bins.
Write
\[
    a_{t,k}=\lfloor s_{t,k}\rfloor,
    \qquad
    f_{t,k}=s_{t,k}-a_{t,k}.
\]
The corresponding interpolation assigns weights $1-f_{t,k}$
and $f_{t,k}$ to bins $a_{t,k}$ and
$\min(a_{t,k}+1,q-1)$, respectively. In the sparse four-read implementation, we retain only the
selected write route. Its straight-through weight is
\[
    \omega_t
    = \prod_{k=1}^{D}
      \left[
        1+(1-f_{t,k})
          -\operatorname{sg}(1-f_{t,k})
      \right],
\]
where $\operatorname{sg}$ denotes stop-gradient.
Numerically, $\omega_t=1$, so each token writes to exactly
one profile during both training and inference. During
backpropagation, the write weight supplies a surrogate gradient
to the hash parameters while the discrete address is held fixed.
This is a one-sided, biased straight-through estimator, rather
than differentiation through the discrete bin index.

We also encourage balanced coordinate occupancy. Let
$\bar p_{k,b}$ be the average soft interpolation mass assigned
to bin $b$ of coordinate $k$ over distant write positions.
The auxiliary loss is
\[
    \mathcal{L}_{\mathrm{bal}}
    = \frac{1}{D}\sum_{k=1}^{D}
      \sum_{b=0}^{q-1}
      \bar p_{k,b}\log\bigl(q\bar p_{k,b}\bigr),
    \qquad
    \mathcal{L}
    = \mathcal{L}_{\mathrm{task}}
      +\beta_{\mathrm{bal}}\mathcal{L}_{\mathrm{bal}}.
\]
This penalizes concentration in individual coordinate bins;
it does not require uniform occupancy of all joint profiles.

\textbf{Query routing and causality.}
After processing the distant block, profile memories are
aggregated into type summaries,
\[
    U_h=\sum_x C_{hx}F_x.
\]
For the single-hyperplane construction, a recent query at
$t=n+i$ hashes its causal representation to a point $u_t$
and chooses a nonzero normalized direction $a_t$ using that
same representation. Its type is
$H_{a_t,a_t^\top u_t}$, with arithmetic over $\mathbb{F}_q$. Consequently, its
long-range support is
\[
    G_{ij}
    = \mathbf{1}
      \left\{
        a_{n+i}^{\top}\operatorname{prof}(j)
        =a_{n+i}^{\top}u_{n+i}
      \right\},
    \qquad j\le n.
\]
This guarantees visibility for matching hash points.
Repeated token identities alone need not produce matching
points when the hash inputs include contextual information.

The four-read variant instead uses an independent learned
scorer $\ell_{i,h}$ computed from $h_{n+i}$. Let
$\mathcal{S}_i=\operatorname{Top4}_h(\ell_{i,h})$. Its read
weights are
\[
    R_{ih}
    =
    \begin{cases}
      \displaystyle
      \frac{\exp(\ell_{i,h})}
           {\sum_{g\in\mathcal{S}_i}\exp(\ell_{i,g})},
      & h\in\mathcal{S}_i,\\[6pt]
      0, & \text{otherwise}.
    \end{cases}
\]
The selected scores receive ordinary gradients through this
softmax; we do not differentiate through the top-four indices.
Unlike the single-hyperplane read, this selector need not
choose a hyperplane through the query's own hash point.

Both variants are causal: every summary contains only
positions $j\le n<n+i$, and the query's routing decision
depends only on its causal representation. Combined with
the causal local branch, the output at position $t$ depends
only on tokens at positions at most $t$.

\paragraph{Memory updates.}
The memory updates are adapted from the recurrent dynamics
of the corresponding backbone, with new payloads routed
to learned profile addresses.

For Mamba-2 extensions, we use an additive memory update: $
F_x^{(j)}
=a_j^G F_x^{(j-1)}
+\mathbf{1}\{x_j=x\}\omega_j k_jv_j^\top$,
where $x_j$ is the learned write profile and $\omega_j$
is an independent sigmoid write gate. We reuse the backbone's input-dependent
retention factor $a_j^G=\exp(-\exp(\theta_A)\Delta_j)$,
where $\theta_A$ is a learned parameter and $\Delta_j>0$
is the input-dependent time step.

For GDN extensions, we use boundary-transported additive
memory, incorporating the backbone's delta-rule transitions. Let $k_t,q_t$ be the normalized
backbone keys and queries, $v_t$ the values, $\beta_t$
the write gate, and $a_t$ the scalar decay. Define $
A_t=a_t(I-\beta_t k_tk_t^\top)$.
For distant token $j\le n$ with learned profile $x_j$,
the boundary-transported key is
\[
\bar k_j=A_n\cdots A_{j+1}\beta_j k_j,
\qquad
F_x=\sum_{j\le n:\,x_j=x}\bar k_jv_j^\top,
\]
where an empty product is the identity.
Equivalently, starting from zero, $
F_x^{(j)}
=A_jF_x^{(j-1)}
+\mathbf{1}\{x_j=x\}\beta_j k_jv_j^\top$.
Thus every profile undergoes the shared transition
$A_j$, while only the selected profile receives the
new payload. The implementation computes transported
keys and pools them additively, rather than explicitly
updating every profile.

After forming $U_b=\sum_x C_{bx}F_x$, a recent token
$t>n$ uses
\[
\bar q_t=r^{-1/2}(A_t\cdots A_{n+1})^\top q_t,
\qquad
z_t=z_t^{\mathrm{local}}
+\lambda_t\bar q_t^\top\sum_b R_{tb}U_b.
\]
Here $R_{tb}$ contains the four selected softmax read
weights and $\lambda_t$ is a learned sigmoid gate. 

\subsection{Experiments}
\label{app:experiments}

\textbf{Exact-support match.} Both subset routing and Multi-key subset recall tasks are scored by whether the model retrieves exactly the requested items. A query counts as correct only if every requested item is present in the output and every unrequested item is absent. In subset routing the output has one channel per marked position, and a channel counts as present when its magnitude exceeds $1\%$ of the largest payload in that example. In multi-key subset recall the output has one raw logit $z_i$ per key--payload pair $i\in\{1,\dots,8\}$, with no sigmoid or softmax applied before thresholding. A pair counts as present when $z_i>0.5$, i.e.\ when its sigmoid probability exceeds $\sigma(0.5)\approx0.62$, and the query is correct when $\{i: z_i>0.5\}$ equals the requested set. Training uses the matching per-pair binary cross-entropy, $-\frac18\sum_i\big[y_i\log\sigma(z_i)+(1-y_i)\log(1-\sigma(z_i))\big]$ averaged over answer rows, with a $k$-hot target $y$. Each evaluation draws $2{,}048$ fresh queries (8 batches of 32 sequences, 8 answer rows each).

\textbf{Subset routing.}
Payloads and requested subsets are resampled every batch, giving roughly $10^5$ examples per run against at most $2^6$ distinct requests. Reported values are the mean of three seeds. 

\begin{table}[H]
\small
\setlength{\tabcolsep}{4pt}
  \centering
  \begin{tabular}{lrrrrrr}
  \hline
  Model & $k=1$ & $k=2$ & $k=3$ & $k=4$ & $k=5$ & $k=6$ \\
  \hline
  Softmax & 100.00 (0.00) & 52.21 (2.09) & 28.78 (0.98) & 17.97 (3.47) &
  7.81 (1.28) & 5.27 (0.90) \\
  Mamba-2 & 100.00 (0.00) & 63.09 (9.95) & 32.23 (6.60) & 19.21 (4.80) &
  9.90 (1.64) & 5.60 (0.81) \\
  DeltaNet & 99.93 (0.11) & 76.69 (1.85) & 46.55 (2.35) & 30.21 (4.35) &
  15.76 (0.45) & 9.83 (1.18) \\
  Gated DeltaNet & 100.00 (0.00) & 66.80 (13.44) & 39.78 (5.18) & 24.22
  (6.80) & 16.93 (0.30) & 8.27 (2.14) \\
  Log-Linear & 99.93 (0.11) & 76.69 (1.71) & 46.55 (4.30) & 28.71 (5.40) &
  16.41 (1.55) & 9.77 (0.59) \\
  \rowcolor{gray!12}
  SMat ($d=1$) & 100.00 (0.00) & 77.21 (1.75) & 48.31 (4.56) & 29.88 (3.65)
  & 14.65 (1.35) & 9.96 (1.03) \\
  \rowcolor{gray!12}
  SMat ($d=2$) & 99.93 (0.11) & 99.48 (0.23) & 82.49 (1.77) & 52.93 (1.17) &
  33.07 (1.44) & 17.12 (3.00) \\
  \rowcolor{gray!12}
  SMat ($d=3$) & 99.93 (0.11) & 99.48 (0.23) & 98.18 (0.56) & 89.45 (3.07) &
  70.18 (2.74) & 43.10 (3.85) \\
  \rowcolor{gray!12}
  SMat ($d=4$) & 99.93 (0.11) & 99.54 (0.11) & 98.18 (0.56) & 97.98 (0.98) &
  93.68 (1.08) & 78.71 (2.25) \\
  \hline
  \end{tabular}
  \caption{Exact routing-pattern match (\%) on subset routing, by the
  number $k$ of marked positions. $T=1024$, one layer, 1500 steps, batch 64,
  learning rate 0.003; payloads are resampled each batch and evaluation
  uses fresh ones. Mean (std) over 3 seeds. SMat performs better at higher $k$ as $d$ increases.\looseness=-1}
  \label{tab:routing}
  \end{table}

The causal mask has VC dimension one because its row supports form a nested family, and softmax attention on the mask can assign substantially different weights to keys within an allowed prefix. The baseline results in Table \ref{tab:routing} therefore describe performance under the evaluated architecture, optimization, and scoring protocol.

\paragraph{Verification of end-to-end cost.} Figure~\ref{fig:prefill} times the prefill of a single SMat layer. Here we measure the full cost of the models that produce these subset routing results. Namely, we analyze training steps (forward pass, backward pass and optimizer update), prefill, and per-token decoding, together with peak memory and decoding cache. All measurements use one NVIDIA H200 GPU with PyTorch~2.11 and Triton~3.7.1, random weights and random inputs (step time depends on shapes, not on trained values). We report the median of at least ten timed steps after warm-up. 

We note some critical testing implementation details. The SMat prefill is reported with custom Triton kernels, but SMat training was done through its Pytorch path (we did not develop the kernels for backwards pass). Softmax attention uses PyTorch SDPA, which in fp32 runs the memory-efficient kernel. The recurrent baselines appear twice. The upper block runs them on optimized kernels: Mamba-2 on the SSD kernel of \texttt{mamba\_ssm}, and DeltaNet and Gated DeltaNet on the chunked kernels of \texttt{flash-linear-attention}, which agree with the reference forms to within $0.4\%$ relative error. The DeltaNet kernel does not accept fp32 inputs, so we ran it on bf16. Log-Linear uses the upstream kernels. The lower block runs the reference PyTorch implementations, as trained for Table~\ref{tab:routing}. SMat decodes with the cached decoder of Section~\ref{subsec:decoding}; softmax decodes from a preallocated KV cache, and the recurrences decode one step at a time. Before timing, every decoder is checked against the full forward pass. The largest relative error in fp32 is $5.3\times10^{-4}$, and the SMat decoder reproduces the prefill rows to $10^{-16}$ in fp64.

\begin{table}[h!]
\centering
\small
\setlength{\tabcolsep}{4pt}
\begin{tabular}{lrrrrrr}
\toprule
& \multicolumn{2}{c}{Train (ms/step)} & \multicolumn{2}{c}{Prefill (ms)} & Decode & Cache \\
\cmidrule(lr){2-3}\cmidrule(lr){4-5}
Model & 16K & 64K & 64K & 256K & (ms/token) & (MB) \\
\midrule
SMat ($d=1$) & 12.9 & 48.1 & 11.5 & 45.6 & 0.25 & 0.5 \\
SMat ($d=2$) & 14.1 & 51.6 & 11.3 & 44.5 & 0.25 & 99 \\
SMat ($d=3$) & 25.1 & 107.0 & 18.3 & 63.8 & 0.25 & 774 \\
SMat ($d=4$) & 192.8 & 1717.9 & 168.6 & 725.4 & 0.26 & 3439 \\
Gated SMat ($d=2$) & 92.7 & 988.9 & 43.1 & 211.3 & 0.34 & 99 \\
Gated SMat ($d=3$) & 104.0 & 1044.3 & 70.9 & 316.9 & 0.34 & 774 \\
\midrule
Softmax (SDPA) & 204.4 & 3147.9 & 836.5 & 13410.2 & 12.90 & 4295 \\
Mamba-2 (SSD kernel) & 8.6 & 29.1 & 7.2 & {--}$^{a}$ & 0.22 & 0.3 \\
DeltaNet (FLA, bf16) & 8.9 & 31.7 & 11.7 & 46.5 & 0.24 & 0.3 \\
Gated DeltaNet (FLA) & 11.3 & 41.9 & 13.8 & 54.7 & 0.31 & 0.3 \\
Log-Linear (upstream) & 21.9 & 97.8 & 22.8 & OOM$^{b}$ & {--} & {--} \\
\midrule
\multicolumn{7}{l}{\emph{Reference implementations, as trained for Table~\ref{tab:routing}}} \\
Mamba-2 & 66.0 & 683.0 & 36.3 & 175.2 & 0.22 & 0.3 \\
DeltaNet & 250.7 & 3230.4 & 85.4 & 396.4 & 0.24 & 0.3 \\
Gated DeltaNet & 279.0 & 3377.7 & 121.7 & 534.4 & 0.31 & 0.3 \\
Log-Linear & OOM$^{c}$ & {--} & 340.7$^{d}$ & {--} & {--} & {--} \\
\bottomrule
\end{tabular}
\caption{End-to-end cost of the subset-routing models: batch 64, fp32. SMat prefill uses the Triton kernels. Decode latency and cache size are at context length 256K. Optimized and reference recurrences share the same one-step decoder. $^{a}$
The SSD kernel exceeds a CUDA launch limit at batch 64 and length 256K. $^{b}$The upstream port builds a dense $T\times T$ level table (512\,GiB at 256K). $^{c}$The reference Log-Linear builds the same dense table and runs out of memory in tr
aining at 16K. $^{d}$Value at 16K.}
\label{tab:e2e-routing}
\end{table}

Table~\ref{tab:e2e-routing} sheds light on SMat's improvements over other attention variants. Up to $d=3$, the hard route SMat is linear in time from end to end. From 16K to 64K its training step grows by $3.7$--$4.3\times$,while softmax attention's grows by $15.4\times$. Prefill grows by $3.5$--$4.0\times$ from 64K to 256K, against $16\times$ for softmax. Per-token decoding is flat at $0.25$\,ms from 1K to 256K, and $0.22$\,ms at batch 1. Against softmax attention this means $29$--$65\times$ faster training steps at 64K, $210$--$300\times$ faster prefill at 256K, $51\times$ faster decoding at 256K, and a cache $5.5\times$ ($d=3$) to $43\times$ ($d=2$) smaller than the KV cache. Against optimized linear-time kernels, SMat's Triton prefill is comparable: at 256K it matches DeltaNet and is faster than Gated DeltaNet. SMat training runs on unfused PyTorch operations and is $1.2$--$1.8\times$ slower than those kernels at 64K for $d\le 2$, and $2.6$--$3.7\times$ slower for $d=3$. It remains faster than the upstream Log-Linear kernels. SMat's decoding cache holds $B+1$ summaries, so it exceeds the constant-size recurrent state. At short lengths it can also exceed the KV cache: at 1K, $d=3$ uses 36\,MB against 17\,MB for softmax, and the order reverses by 4K. At $d=4$ the $O(T^{2-3/d})$ long-range term dominates at these lengths. SMat is then slower than the linear-time kernels, but still $1.8\times$ faster than softmax in training at 64K and $18\times$ faster in prefill at 256K. Gating adds only linear work (Lemma~\ref{thm:gated}). Our gated causal scan, however, is an unfused PyTorch loop over chunks, so gated training is about $20\times$ slower than ungated at 64K, while gated prefill and decoding stay within $2.5$--$5\times$ and $1.4\times$ of ungated.

\textbf{Multi-key subset recall.}
Keys and payloads are placed at random positions in the distant block, so the task cannot be solved positionally. The hyperplane read is used throughout, and the reported metric is exact-support match: a query counts as correct only when the retrieved support equals the requested key set. \emph{LSH bucketing} \citep{kitaev2020reformerefficienttransformer} hashes keys into the same cells, but a query reads only its own cell rather than a hyperplane of cells. \emph{Memory-matched linear attention} tests whether SMat's advantage is simply the larger state it keeps, since associative recall in efficient models is known to be limited by state size \citep{arora2023zoologymeasuringimprovingrecall}.

\begin{table}[h]
\centering\scriptsize\setlength{\tabcolsep}{7pt}
\begin{tabular}{lrrrrrrrrr}
\toprule
 & \multicolumn{3}{c}{$k=1$} & \multicolumn{3}{c}{$k=2$} & \multicolumn{3}{c}{$k=3$} \\
\cmidrule(lr){2-4} \cmidrule(lr){5-7} \cmidrule(lr){8-10}
Model & s0 & s1 & s2 & s0 & s1 & s2 & s0 & s1 & s2 \\
\midrule
Softmax & 100.0 & 7.9 & 25.0 & 99.9 & 0.0 & 100.0 & 99.9 & 0.0 & 99.5 \\
Mamba-2 & 0.0 & 0.0 & 0.0 & 0.0 & 0.0 & 0.0 & 0.0 & 0.0 & 0.0 \\
DeltaNet & 0.0 & 0.0 & 0.0 & 0.0 & 0.0 & 0.0 & 0.0 & 0.0 & 0.0 \\
Gated DeltaNet & 0.0 & 0.0 & 0.0 & 0.0 & 0.0 & 0.0 & 0.0 & 0.0 & 0.0 \\
Log-Linear & 0.0 & 0.0 & 0.0 & 0.0 & 0.0 & 0.0 & 0.0 & 0.0 & 0.0 \\
Linear attention & 0.0 & 0.0 & 0.0 & 0.0 & 0.0 & 0.0 & 0.0 & 0.0 & 0.0 \\
Linear attention (memory matched) & 0.0 & 0.0 & 0.0 & 0.0 & 0.0 & 0.0 & 0.0 & 0.0 & 0.0 \\
LSH bucketing (121 buckets) & 100.0 & 100.0 & 100.0 & 0.7 & 0.4 & 0.3 & 0.0 & 0.0 & 0.0 \\
LSH bucketing (125 buckets) & 100.0 & 100.0 & 100.0 & 1.8 & 1.2 & 2.1 & 0.0 & 0.0 & 0.0 \\
SMat ($d=2$) & 100.0 & 100.0 & 100.0 & 11.7 & 14.0 & 16.0 & 0.9 & 1.6 & 1.4 \\
SMat ($d=3$) & 100.0 & 100.0 & 100.0 & 99.9 & 99.9 & 99.7 & 35.5 & 33.1 & 34.0 \\
SMat ($d=4$) & 100.0 & 100.0 & 100.0 & 100.0 & 99.9 & 99.9 & 99.0 & 99.0 & 99.0 \\
\bottomrule
\end{tabular}
\caption{Per-seed exact-support accuracy (\%) behind Table~\ref{tab:mkar}, after 12K steps on 2{,}048 fresh queries per run. SMat rows use rule-chosen hyperplane directions (see text).}
\label{tab:mkar-seeds}
\end{table}

To further examine these results, a network that ignores the query does best by predicting the base rate $k/8$ for every pair. That constant predictor has loss $H(k/8)=0.377$, $0.562$ and $0.662$ nats for $k=1,2,3$, and its logit $\log\frac{k}{8-k}<0.5$, so it marks every pair absent. More generally, a constant output is all-absent or all-present under any threshold and therefore never equals a $k$-hot target, so it scores exactly $0$ whether the threshold is placed on logits or on probabilities. All $54$ linear-baseline runs end within $0.0012$ nats of $H(k/8)$ (Figure~\ref{fig:mkar-sanity}). Their zeros therefore come from models that never began to use the query. For reference, guessing $k$ pairs uniformly at random scores $1/\binom8k=12.5\%$, $3.6\%$ and $1.8\%$.

Table~\ref{tab:mkar-seeds} lists every run behind Table~\ref{tab:mkar}, and Figure~\ref{fig:mkar-sanity} shows the training loss and exact-support accuracy of each seed. Softmax outcomes are bimodal. Each run starts on the constant-prior plateau with some runs leaving after $6$K and $10$K steps and reaching $\geq99\%$ by the next evaluation. Seed~1 never leaves the plateau at $k=2,3$; at $k=1$, seeds~1 and~2 had begun to leave it when training stopped. The large standard deviations of the softmax row therefore record whether a seed escaped the plateau within the budget. The other warmup and learning-rate settings we tried for softmax did no better (Table~\ref{tab:mkar-softmax-sweep}).

It is also important to note the difference between the rule-chosen and learned hyperplane schemas for the SMat matrix in this experiment. In the SMat rows of Table~\ref{tab:mkar}, the hyperplane direction at each answer row is set, during training and evaluation, by a fixed rule that picks a direction on which the hashed cells of all $k$ requested keys agree whenever one exists. The baselines have no equivalent mechanism. This isolates what the mask can express. When a linear head predicts the direction from the query instead (seed~0 only), SMat scores $99.9$, $60.1$ and $44.6\%$ at $d=3$ and $100.0$, $74.2$ and $59.2\%$ at $d=4$ for $k=1,2,3$. At $d=2$ the hyperplanes are single points, so there is only one hyperplane direction and the two schemas coincide.

\begin{table}[h]
\centering\small
\begin{tabular}{lccc}
\toprule
Softmax schedule, LR & $k=1$ & $k=2$ & $k=3$ \\
\midrule
no warmup, $3\times10^{-4}$ (reported) & 100.0 / 7.9 / 25.0 & 99.9 / 0.0 / 100.0 & 99.9 / 0.0 / 99.5 \\
1K-step warmup, $10^{-4}$ & 100.0 / 0.0 / 13.4 & 0.0 / 0.0 / 0.0 & 0.0 / 0.0 / 0.0 \\
1K-step warmup, $3\times10^{-4}$ & 100.0 / 1.5 / 13.4 & 0.0 / 0.0 / 2.8 & 99.8 / 0.0 / 0.0 \\
1K-step warmup, $10^{-3}$ & 0.0 / 0.0 / 100.0 & 0.0 / 0.0 / 0.0 & 0.0 / 0.0 / 0.0 \\
\bottomrule
\end{tabular}
\caption{Softmax on multi-key subset recall at $T=1024$: exact-support accuracy (\%) for seeds 0 / 1 / 2 under each schedule tried. A run that leaves the constant-prior plateau reaches $\geq 99\%$; the runs between 1\% and 25\% had begun leaving it when training stopped. The reported row is the best schedule.}
\label{tab:mkar-softmax-sweep}
\end{table}

\begin{table}[h]
\centering\small
\begin{tabular}{lrrr}
\toprule
Model & $T=128$ & $T=256$ & $T=1024$ \\
\midrule
Softmax & 97.5 & 99.5 & 100.0 \\
Mamba-2 & 99.2 & 0.0 & 0.0 \\
DeltaNet & 100.0 & 100.0 & 0.0 \\
Gated DeltaNet & 100.0 & 100.0 & 0.0 \\
Log-Linear & 99.9 & 93.6 & 0.0 \\
Linear attention & 94.8 & 35.6 & 0.0 \\
\bottomrule
\end{tabular}
\caption{Positive control: exact-support accuracy (\%) for $k=1$, seed 0, with the harness, metric, model and 12K-step budget of Table~\ref{tab:mkar} and only the context length changed (8 pairs throughout).}
\label{tab:mkar-posctrl}
\end{table}

\begin{figure}[h!]
\centering
\includegraphics[width=0.9\linewidth]{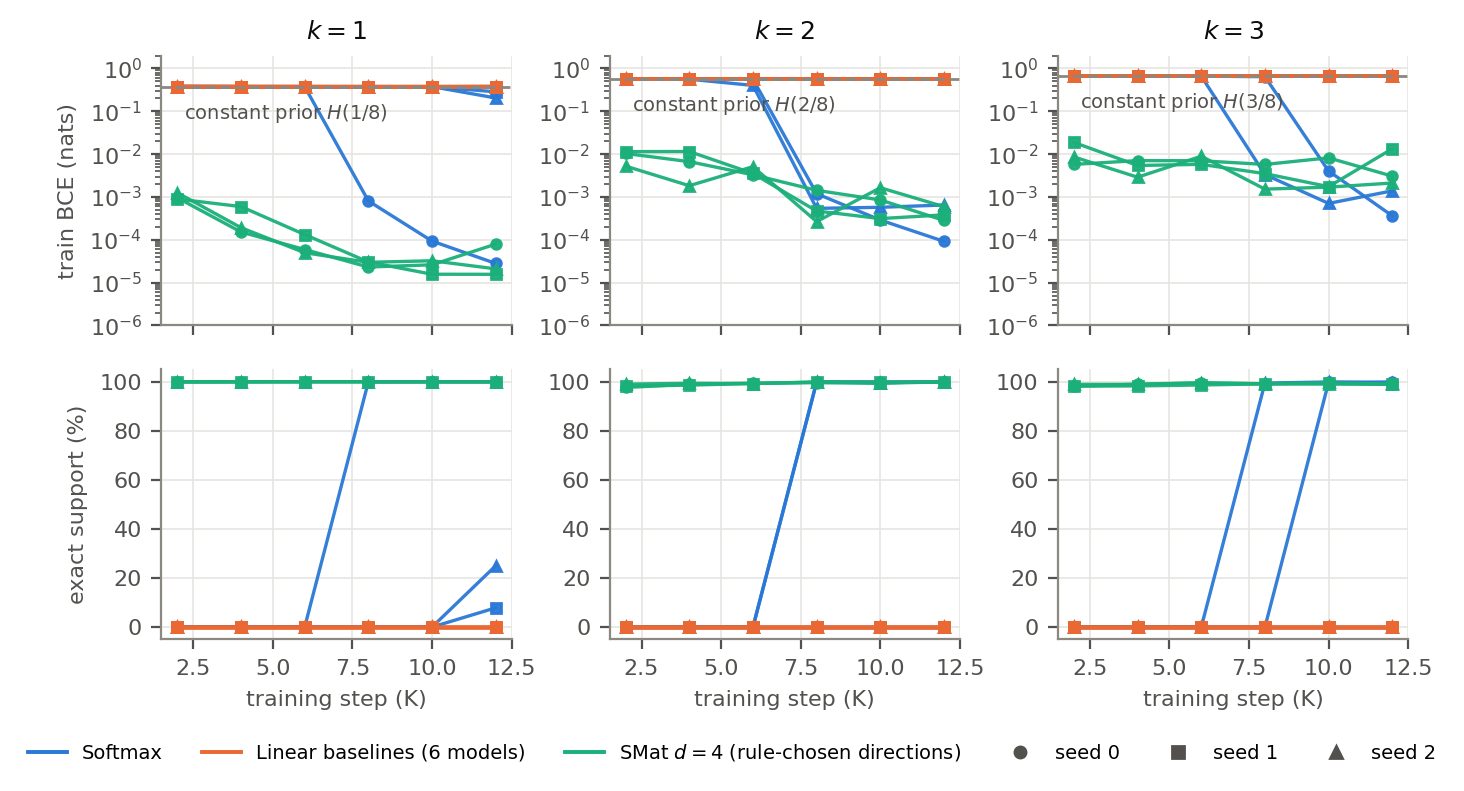}
\caption{Multi-key subset recall at $T=1024$, every seed. Top: training binary cross-entropy of the minibatch at each evaluation step (log scale); the dashed line is the constant-prior loss $H(k/8)$, and all $18$ linear-baseline runs per panel lie on it. Bottom: exact-support accuracy on $2{,}048$ fresh queries, evaluated every $2$K steps.}
\label{fig:mkar-sanity}
\end{figure}

\begin{table}[h]
\centering
\caption{Mamba-2 SMat results across three seeds. Mean (std).}
\label{tab:mamba-smat-3seeds-pop}
\begin{tabular}{lccc}
\toprule
Model & \(k=1\) & \(k=2\) & \(k=3\) \\
\midrule
Mamba-2     & 0.00 (0.00) & 0.00 (0.00)  & 0.00 (0.00) \\
SMat (\(d=2\))    & 99.48 (0.72) & 2.51 (3.54) & 0.00 (0.00) \\
SMat (\(d=3\))    & 96.43 (4.65) & 84.29 (19.86) & 0.00 (0.00) \\
SMat (\(d=4\))    & 100.00 (0.00) & 99.98 (0.02) & 99.90 (0.14) \\
\bottomrule
\end{tabular}
\end{table}

Table~\ref{tab:mamba-smat-3seeds-pop} shows the results of a learned Mamba-2--SMat on multi-key subset recall. The baseline Mamba-2 achieves zero exact-support accuracy under this training setup, whereas SMat with $d=2$ succeeds primarily at $k=1$, $d=3$ substantially improves performance at $k=2$, and $d=4$ achieves near-perfect accuracy across all three settings. This qualitatively matches the original MKAR pattern, with increasing geometric dimension enabling accurate retrieval of larger requested subsets.

\textbf{MQAR.}
Run through the Zoology harness \citep{arora2023zoologymeasuringimprovingrecall} with its data pipeline unchanged: training mixtures of $4$, $8$, $16$, $32$, and $64$ key--value pairs at lengths $64$--$256$ ($100$K examples for the $4$-pair mixture, $20$K for each of the others), and $1000$ held-out examples per mixture. Evaluation batch size is $32$. The GDN backbone uses one head at width $16$ and two heads at widths $32$ and $64$, with value expansion $1$, following
\citet{guo2025loglinearattention}; $32$ epochs is $707$ optimizer steps per epoch. Log-Linear Attention uses the authors' released implementation with the same widths, head configuration, and value expansion as the corresponding backbone, and the same 32-epoch budget.

\textbf{Joint context-key recall (JCKR).}
Models retrieve values from shuffled context-key-value records, with keys
shared across contexts and values sampled uniformly from 16 symbols.
Every context-key pair is queried in random order, with answers masked
and cross-entropy applied only at query positions.
We use $(C,K)\in\{(1,4),(2,8),(8,16),(16,16),(32,16)\}$, yielding
sequence lengths $64$--$3076$, with 36000/2000/4000 train/validation/test
examples per configuration. Two-layer models of width 64 train for
32 epochs using AdamW, learning rate $3\times10^{-3}$, cosine decay,
and batch size 256 (Table~\ref{tab:training}).
We report final-epoch validation accuracy averaged equally across
configurations, with mean and standard deviation over three training seeds.

\begin{table}[h]
\centering
\begin{tabular}{lc}
\hline
Model & Accuracy (\%) \\
\hline
Plain GDN & $53.80\ (2.86)$ \\
GDN + SMat ($d=3$) & $\mathbf{58.92\ (4.22)}$ \\
MoM-derived, profile-count matched & $43.73\ (1.48)$ \\
MoM-derived, matrix-storage matched & $35.10\ (3.82)$ \\
\hline
\end{tabular}
\caption{Shared-key joint recall: final-test accuracy averaged
equally over the five evaluated binding loads. Mean and std reported across three seeds.}
\label{tab:mom_macro}
\end{table}

Furthermore, we compare SMAT with two variants derived from
Mixture-of-Memories (MoM)~\cite{du2026mom}, both using Gated
DeltaNet updates, top-4 routing, and no shared memory. At each sequence length,
the profile-count variant has one independent memory matrix per
SMAT profile, while the matrix-storage variant matches the
combined FP32 storage of SMAT's profile and derived summary
matrices. All memory matrices are $16 \times 16$.
We retain the same data splits, model width, layer and head counts,
and 32-epoch training recipe.

\begin{table}[H]
\centering

\begin{subtable}[t]{0.48\linewidth}
\centering
\small
\setlength{\tabcolsep}{5pt}
\begin{tabular}{
    l
    S[table-format=1.3]
    S[table-format=1.3]
}
\toprule
\textbf{Model} & {\textbf{16K}} & {\textbf{32K}} \\
\midrule
Mamba-2      & 3.752 & 3.803 \\
\rowcolor{gray!12}
SMat $d=2$   & 3.743 & 3.788 \\
\rowcolor{gray!12}
SMat $d=3$   & 3.744 & 3.787 \\
\rowcolor{gray!12}
SMat $d=4$   & 3.741 & 3.788 \\
Softmax      & 3.797 & 3.910 \\
\bottomrule
\end{tabular}
\phantomsubcaption
\label{tab:pg19-300m}
\par\smallskip
(a) 300M training tokens
\end{subtable}
\hfill
\begin{subtable}[t]{0.48\linewidth}
\centering
\small
\setlength{\tabcolsep}{4pt}
\begin{tabular}{
    l
    S[table-format=2.3]
    S[table-format=2.3]
}
\toprule
\textbf{Variant} & {\textbf{GDN}} & {\textbf{Mamba-2}} \\
\midrule
Baseline      & 23.165 & 24.114 \\
\rowcolor{gray!12}
SMat ($d=2$) & 23.097 & 24.117 \\
\rowcolor{gray!12}
SMat ($d=3$) & 23.214 & 24.182 \\
\bottomrule
\end{tabular}
\phantomsubcaption
\label{tab:pg19-750m}
\par\smallskip
(b) 750M training tokens
\end{subtable}

\caption{
PG-19 language modeling results. Left: negative log-likelihood (NLL) per
token after 300M training tokens with eight layers and width 384.
Right: validation perplexity after 750M training tokens at 16K context
(seed 123). Lower is better.
}
\label{tab:pg19-results}
\end{table}
\textbf{PG-19.}
Books are tokenized with the GPT-2 BPE vocabulary and packed into fixed-length windows; all arms see the same data order. The transformer arm is attention plus a $4\times$ MLP with eight heads, sized so that its non-embedding parameter count is comparable to SMat's. An auxiliary load-balancing term with coefficient $0.01$ is applied to the routing hash and annealed over the first $1000$ steps. Evaluation is per-token negative log-likelihood on the PG-19 test split
($100$ books, $5\times10^6$ tokens) at the training context length. Peak memory is roughly $15$\,GB per arm at $16$K and $30$\,GB at $32$K. \footnote{Due to academic compute constraints, each model configuration in Table \ref{tab:pg19-results} was trained with a single seed.}\looseness=-1

\end{document}

%% file: math_commands.tex
\usepackage{amsmath,amsfonts,bm}
\usepackage{amsthm}
\newtheorem{definition}{Definition}
\newtheorem{lemma}{Lemma}

\def\eqref#1{equation~\ref{#1}}

\def\1{\mathbf{1}}

\def\mA{{\mathbf{A}}}

\def\mC{{\mathbf{C}}}
\def\mD{{\mathbf{D}}}
\def\mE{{\mathbf{E}}}
\def\mF{{\mathbf{F}}}
\def\mG{{\mathbf{G}}}
\def\mH{{\mathbf{H}}}

\def\mK{{\mathbf{K}}}
\def\mL{{\mathbf{L}}}
\def\mM{{\mathbf{M}}}
\def\mN{{\mathbf{N}}}

\def\mQ{{\mathbf{Q}}}
\def\mR{{\mathbf{R}}}

\def\mV{{\mathbf{V}}}

\def\mZ{{\mathbf{Z}}}

\DeclareMathAlphabet{\mathsfit}{\encodingdefault}{\sfdefault}{m}{sl}
\SetMathAlphabet{\mathsfit}{bold}{\encodingdefault}{\sfdefault}{bx}{n}

\newcommand{\R}{\mathbb{R}}

\usepackage{booktabs}
\usepackage{algorithm}
\usepackage[noend]{algorithmic}
\usepackage{hyperref}
\usepackage{cleveref}
\usepackage{amssymb}
\usepackage{bbm}
\usepackage{amsthm}